\documentclass[12pt]{article}

\usepackage[english]{babel}

\usepackage[letterpaper,top=2cm,bottom=2cm,left=3cm,right=3cm,marginparwidth=1.75cm]{geometry}
\usepackage{algorithm} % pseudo-code
\usepackage{algpseudocode} % pseudo-code
\usepackage{mathrsfs} % \mathscr
\usepackage{booktabs} % table
\usepackage{tabularx} % table
\usepackage{float} % h and H
\usepackage{multirow} % multi-row table
\usepackage{natbib} % citation
\usepackage{cancel} % cross out
\usepackage{balance}
\usepackage{hyperref}
\usepackage{lineno}
\usepackage{subcaption} % multiple plots
\usepackage{tikz-setup} % put all tikz setups here 
\usepackage{hyperref}
\usepackage{amsmath}
\usepackage{amsthm}
\usepackage[all]{xy}
\usepackage{adjustbox}
\usepackage{stmaryrd}
\usepackage{subcaption}
\newcommand{\tomato}{\texttt{ToMATo}}
\DeclareMathOperator*{\argmax}{arg\,max}
\DeclareMathOperator*{\argmin}{arg\,min}
\DeclareMathOperator{\PS}{PS}
\DeclareMathOperator{\TR}{TR}
\DeclareMathOperator{\PTR}{PTR}
\DeclareMathOperator{\supp}{Supp}
\newcommand{\obs}{\mathbf{x}}

\newcommand{\Set}{\mathcal{S}}
\newcommand{\INs}{\mathcal{X}}
\newcommand{\OUTs}{\mathcal{Y}}
\newcommand{\real}{\mathbb{R}}
\newcommand{\nat}{\mathbb{N}}
\newcommand{\Prob}{\mathbb{P}}

\usepackage{mathtools}
\usepackage{amsfonts} % mathbb
\usepackage{amssymb} % transpose sign
\usepackage{bm} % bold symbols
\usepackage{dsfont} % for the indicator

\newtheorem{theorem}{Theorem}
\newtheorem{result}{Result}
\newtheorem{definition}{Definition}
\newtheorem{lemma}{Lemma}

\DeclarePairedDelimiterX{\infdivx}[2]{(}{)}{%
 #1\,\delimsize\|\,#2%
}

\usepackage{transparent}

\usepackage{ulem}

\hypersetup{
 colorlinks = true,
 linkbordercolor = {white},
 citecolor = {blue},
}

\title{Pool-Based Active Learning with Proper Topological Regions}

\author{Lies Hadjadj$^\dagger$, Emilie Devijver$^\dagger$, Remi Molinier$^\ddagger$, Massih-Reza Amini$^\dagger$\\ 
\{Firstname.Lastname\}@univ-grenoble-alpes.fr\\
$^\dagger$ Computer Science Laboratory (LIG)\\
$^\ddagger$ Department of Mathematics (IF)\\
University of Grenoble Alpes, France}

\date{}
\begin{document}

\maketitle

\begin{abstract}Machine learning methods usually rely on large sample size to have good performance, while it is difficult to provide labeled set in many applications. Pool-based active learning methods are there to detect, among a set of unlabeled data, the ones that are the most relevant for the training.
We propose in this paper a meta-approach for pool-based active learning strategies in the context of multi-class classification tasks based on Proper Topological Regions. PTR, based on topological data analysis (TDA), are relevant regions used to sample cold-start points or within the active learning scheme. The proposed method is illustrated empirically on various benchmark datasets, being competitive to the classical methods from the literature. 
\end{abstract}

\section{Introduction}
In recent years, machine learning has found gainful applications in diverse domains, but it still has a heavy dependence on expensive labeled data: Advances in cheap computing and storage have made it easier to store and process large amounts of unlabeled data, but the labeling often needs to be done by humans or using costly tools. Therefore, there is a need to develop general domain-independent methods to learn models effectively from a large amount of unlabeled data at the disposal, along with a minimal amount of labeled data: this is the framework of semi-supervised learning. Active learning aims explicitly to detect the observations to be labeled to optimize the learning process and efficiently reduce the labeling cost.  
The primary assumption behind active learning is that machine learning algorithms could reach a higher level of performance while using a smaller number of training labels if they were allowed to choose the training dataset \citep{settles09}. The most common active learning approaches are pool-based methods \citep{Lew94} based on a set of unlabeled observations. First, some points are labeled to train a classification model, and then, at each iteration, we choose unlabeled examples to query based on the predictions of the current model and a predefined priority score. These approaches show their limitations in low-budget regime scenarios because they need a sufficient budget to learn a weak model \citep{pour21}. The literature has shown that for active learning to operate in a low-budget regime successfully, there is a need to introduce a form of regularization in training \citep{guyon11} usually found in other sub-domains, such as semi-supervised learning or self-learning \citep{Chapman06}. Another line of work shows that the choice of the initial seed set in these approaches significantly impacts the end performance of their models \citep{Hu10, Chen22}, also known as the cold-start problem in active learning.

We close the gap in this paper by providing a theoretically founded meta-approach for pool-based active learning based on concepts from topological data analysis (TDA) to improve performance in a low-budget regime and avoid the cold-start problem. TDA aims to extract information on the structure of the data by examining its topological properties \citep{Edel10}, the insight being that nontrivial topologies should be exploited to improve data analysis \citep{carlsson12}. This structure can be detected by flexible tools based on algebraic topology, for example, using persistent homology based on Rips complexes \citep{Haus95}: topological information is then encoded with persistence modules and diagrams \citep{Edel10}. It has already shown impressive results in machine learning \citep{Rieck20,Jia21, Kri21}, especially for clustering. Many recent papers have benefited from these topological insights to understand the structure of the data: \citet{Sin07} use persistence homology to extract molecular topological fingerprints (MTFs) based on the persistence of molecular topological invariants, \citet{Lum13} use topological persistence to efficiently encode fMRI datasets, \citet{Gar20} use persistence homology to automatically extract interpretable features from meta-organic datasets in order to predict methane and carbon dioxide adsorption levels for different materials, among others, and \citet{Li20} also make use of topological persistence in order to actively estimate the homology of the Bayes decision boundary, the resulted module is then used to do model selection from several families of classifiers.

In this paper, we propose to extend \tomato{} \cite{Chazal13}, a persistent-based clustering algorithm that respects the underlying topology, to detect \emph{proper topological regions} where one can safely propagate labeling (assuming that the clustering is coherent with the metric). More precisely, our approach is based on the following:
\begin{itemize}
 \item the introduction of proper topological regions using the $\sigma$-Rips graph based on an adaptive threshold function and the extension of \tomato's theoretical guarantees to the $\sigma$-Rips graph; 
 \item the use of proper topological regions in a zero-shot learning method and a pool-based active learning scheme.
\end{itemize}
 This is illustrated in an empirical study with several active learning strategies which shows that our approach for zero-shot learning and pool-based active learning improves over classical methods on several datasets. 
 
The remainder of the paper is organized as follows. Section \ref{sec:SOTA} describes the related work. The framework is introduced in Section \ref{sec:framework}. The method is developed in Section \ref{sec:method}; and illustrated in Section \ref{sec:experiments}. Finally, Section \ref{sec:concl} concludes the paper.

\section{Related literature}\label{sec:SOTA}
Different attempts have been made to reduce the annotation burden of machine learning algorithms. We can refer to the remarkable advances made in semi-supervised learning \citep{AminiUsunier15, Bert19}. These methods take as input a small set of labeled training data together with a large number of unlabeled examples. They introduce a form of consistency regularization to the supervised loss function by applying data augmentation using unlabeled observations \citep{Chapman06}.
The most commonly known pool-based strategies are uncertainty sampling \citep{Lew94, zhu08}, margin sampling, and entropy sampling strategies \citep{settles09}. Some proposed strategies rely on the query-by-committee approach \citep{yan11, lakshmi2016}, which learns an ensemble of models at each round. Query by bagging and query by boosting are two practical implementations of this approach that use bagging and boosting to build the committees \citep{Abe98}. There has been exhaustive research on how to derive efficient disagreement measures and query strategies from a committee, including vote entropy, consensus entropy, and maximum disagreement \citep{settles09}, whereas \citet{Ali14} introduces model selection for a committee. Some research focuses on solving a derived optimization problem for optimal query selection, e.g. in \citet{RoyM01} they use Monte Carlo estimation of the expected error reduction on test examples. In contrast, other strategies employ Bayesian optimization on acquisition functions such as the probability of improvement or the expected improvement \citep{garnett22}, and in \citet{Auer02}, the authors propose to cast the problem of selecting the most relevant active learning criterion as an instance of the multi-armed bandit problem. Aside from the pool-based setting, in the stream-based setting \citep{Lughofer12, Baram04}, each unlabeled sample is given to the learner individually, and he queries its label if he finds it helpful.

Recent advances in active learning propose enhancing the pool-based methods by extracting knowledge from the distribution of unlabeled examples \citep{Bon11}. \citet{Fab18} propose to use clustering of unlabeled examples to boost the performance of pool-based active learners, with the expert annotating at each iteration cluster rather than single examples. Such a strategy effectively reduces the annotation effort, assuming that the cost of cluster annotation is comparable to single example labeling, as used in \citet{citovsky2021batch} to operate on large-scale data. Similarly, \citet{Krem15} proposes to combine clustering with Bayesian optimization in the stream-based setting. \cite{Yu17} propose a two-stage clustering constraint in the active learning algorithm, a first exploration phase to discover representative clusters of all classes, and a post-clustering reassignment phase where the learner is constrained on the initial clusters found at the first stage. Clustering methods also show promising results for addressing the cold-start problem in pool-based active learning strategies \citep{Hu10, Chen22}. In \citet{Urner13} authors propose a procedure for binary domain feature sets to recover the labeling of a set of examples while minimizing the number of queries. They show that this routine reduces label complexity for training learners.
Recently, many studies have explored the use of clustering/segmentation for active sample selection in real applications \cite{andresini_2023, thoreau_2022}. 

\section{Framework and topological considerations}
\label{sec:framework}
We introduce in this section the framework of active learning and the topological background needed to develop the proposed method. First, we introduce the framework and the main topological notions. Then we define the persistence and upper-star filtrations. Finally, we provide a comparison of persistence diagrams for Rips graph and $\sigma$-Rips graph, which allows us to extend \tomato{} results to the $\sigma$-Rips graph. 

\subsection{Framework and notations}

We consider a multi-class classification problem such that the input space is $\INs \subset \real^m$ and the output space $\OUTs=\{1,\dots,c\}$ is a set of unknown classes of cardinal $c \in \nat, c\geq 2$. Let $d$ be a fixed distance on $\mathbb{R}^m$. 
 %, and the pair $(\INs,d)$ is a metric space, where $d \colon \INs \times \INs \to [0,\infty)$ is a fixed and known distance metric. 
In pool-based active learning, we observe a sample set $\Set_\obs = \{\obs_i\}_{i=1}^n$ coming from an unknown marginal distribution $\Prob$, and we have access to an oracle $\mathcal{O}: \INs \rightarrow \OUTs$ that can provide the true label $y_i$ for every observation $\obs_i$, for $1\leq i \leq n$ at some (expensive) cost. 
We denote $\Set = \{(\obs_i, y_i)\}_{i=1}^n$ the labeled data sample of size $n$, which we do not have access to, generated by some unknown joint distribution over $\INs \times \OUTs$.

In our method, as generally is the case in classification algorithms, we assume that close samples (with respect to $d$) are associated with similar labels, also known as the \textit{smoothness assumption}.
In that setting, one can consider neighborhood graphs on the unlabeled sample $\Set_\obs$. A graph is denoted as a couple $(V,E)$ with $V$ the set of vertices, and $E$ the set of edges. For our purpose, we use a neighborhood graph induced by the metric $d$ on $\INs$.
\begin{figure}[t]
 \includegraphics[scale=0.4]{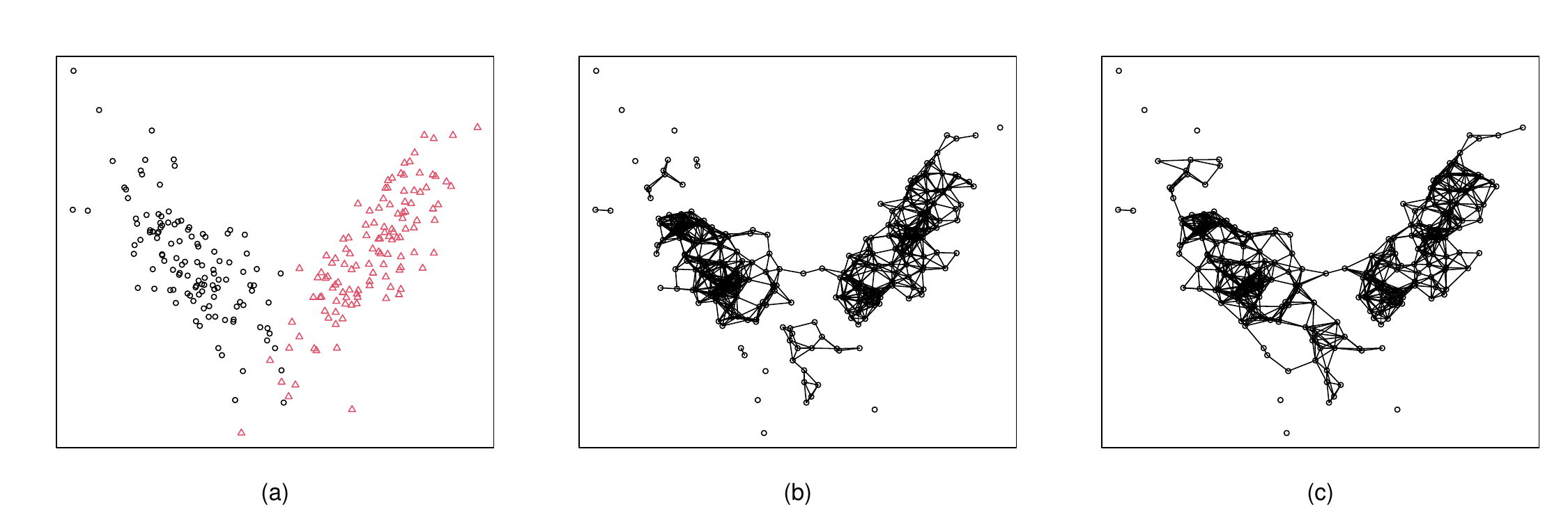}
 \caption{(a) A sample of 240 points, generated from a mixture of two bivariate Gaussian distributions. Colors represent the true classes. (b) Associated Rips graph as defined in Def. \ref{def:rips} with $\delta = 0.5$. (c) Associated $\sigma$-Rips graph as defined in Def. \ref{def:sigma_rips}, using the parametric form given in Eq. \eqref{sigma_delta_def} with $\delta=0.5$, $r=1.08$ and $t = 1/5$. }
 \label{fig:rips_graph}
\end{figure}

\begin{definition}[Rips graph] \label{def:rips}
Given a finite point cloud $\Set_\obs = \{\obs_i\}_{i=1}^n$ from a metric space $(\INs, d)$ and $\delta \geq 0$, the \emph{Rips graph} $R_\delta(\Set_\obs)$ is the graph with set of vertices $\Set_\obs$ and whose edges correspond to the pairs of points $(\obs_i, \obs_j) \in \Set_\obs^2$ such that $d(\obs_i,\obs_j) \leq \delta$.
\end{definition}

Rips graphs, or more generally Rips complexes \citep{Chazal14}, are classical in topology and are classically used in TDA, in particular with persistent homology.
However, class similarity might be different over the metric space. For example, lower is the density, weaker is the chance to detect a structure within points.
Consequently, we need to generalize the definition of the Rips graph to take into account such cases, namely the $\sigma$-Rips graph $R_{\sigma(\cdot)}(\Set_\obs)$ for an adaptive threshold function $\sigma$. 

\begin{figure}[t]
 \includegraphics[scale=0.4]{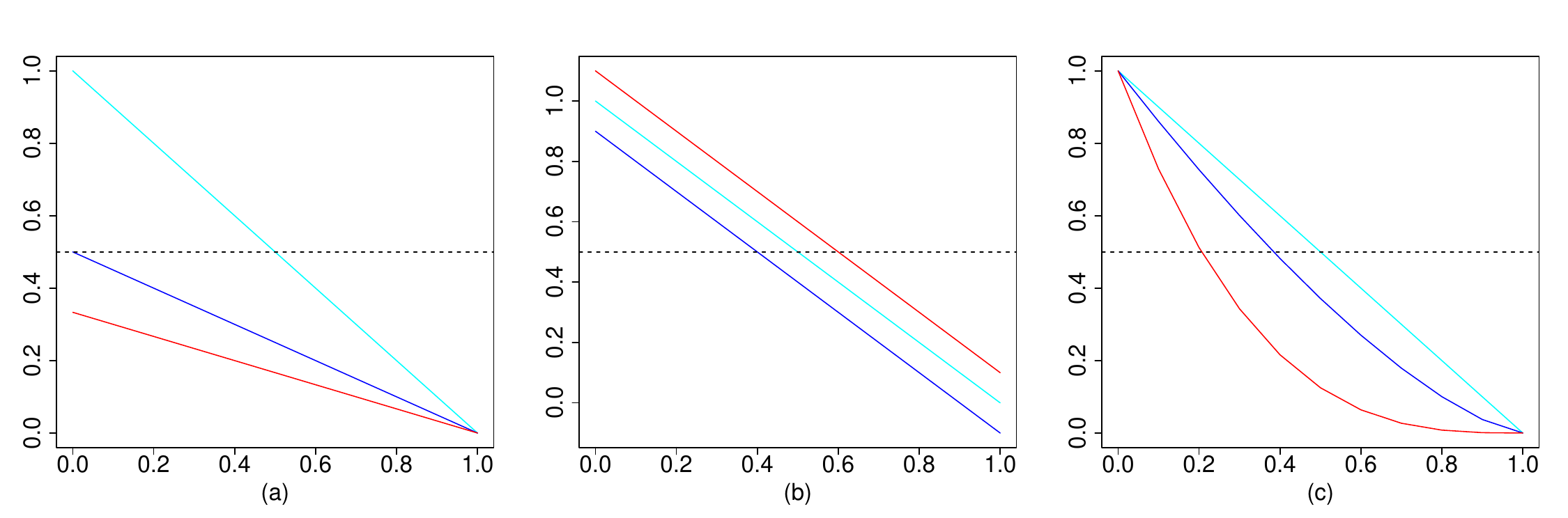}
 \caption{Representation of $s : u \mapsto \delta(r-u)^{1/t}$ as a proxy of the parametric form of $\sigma$ given in Eq. \eqref{sigma_delta_def}, varying the parameters. By default, all the parameters are fixed to 1. We vary in (a) $\delta \in \{1, 0.5, 1/3\}$, in (b) $r \in \{0.9,1, 1.1\}$, in (c) $t \in \{1, 0.7, 1/3\}$. Dashed line is for a constant threshold function $\sigma=0.5$. }
 \label{fig:sigma}
\end{figure}
\begin{definition}[$\sigma$-Rips graph] \label{def:sigma_rips}
Given a finite point cloud $\Set_\obs = \{\obs_i\}_{i=1}^n$ from a metric space $(\INs, d)$ and a real-valued function $\sigma \colon \INs^2 \to \real^*_+$, the \emph{$\sigma$-Rips graph} $R_{\sigma(\cdot)}(\Set_\obs)$ is the graph with set of vertices $\Set_\obs$ and whose edges correspond to the pairs of points $(\obs_i, \obs_j) \in \Set_\obs^2$ such that $d(\obs_i,\obs_j) \leq \sigma(\obs_i,\obs_j)$.
\end{definition}
Those two notions of neighborhood graph are illustrated in Figure \ref{fig:rips_graph} to understand the differences. When the density is lower (few points), the $\sigma$-Rips graph is more connected, to enforce the structure to appear. 

The $\sigma$-Rips graph can be seen as a generalization of the Rips graph, which considers constant threshold function, or as a $\delta$-Rips graph on the non-metric space $(\INs, \hat{d})$, with
\begin{align}
 \hat d \colon \INs \times \INs & \longrightarrow \real^+\nonumber\\
 (\obs, \obs') & \longrightarrow \frac{\delta d(\obs, \obs')}{ \sigma(\obs, \obs')}.\label{dchapeau}
\end{align}
Most of the topological properties of Rips graphs on metric spaces are true for Rips graphs on non-metric spaces, as mentioned in \citet[Section 4.2.5]{Chazal14}.

In this work, we choose the following parametric threshold function:
\begin{align}
\label{sigma_delta_def}
\begin{split}
 \sigma(\cdot;\delta, r, t) \colon \INs \times \INs & \longrightarrow \real^*_+\\
 (\obs, \obs') & \longrightarrow \delta (r-\max{(\Prob(\obs), \Prob(\obs'))})^{\frac{1}{t}},
\end{split}
\end{align}
with $t \in (0, 1]$ and $(\delta, r) \in (\real^*_+)^2$ such that $r > \max_\obs \Prob(\obs)$. The temperature parameter $t$ controls the curvature, the $\max$ term ensures that the function is symmetric. Then, $\delta$ and $r$ are, respectively, dilatation and translation parameters. This parametric form is illustrated in Figure \ref{fig:sigma}. We show in Section~\ref{sec:exp:graphs} that the curve resulting from the best parameters of our function confirms our intuition on the class similarity being a density-aware measure.

\subsection{Persistence and upper-star filtrations}

In order to detect the underlying topology from a point cloud, our method is based on the notion of persistence, and more precisely, persistent homology, a classical tool in TDA \citep{Sin07, Edel10,carlsson12}.

\begin{figure}
\centering
 \includegraphics[scale=0.6]{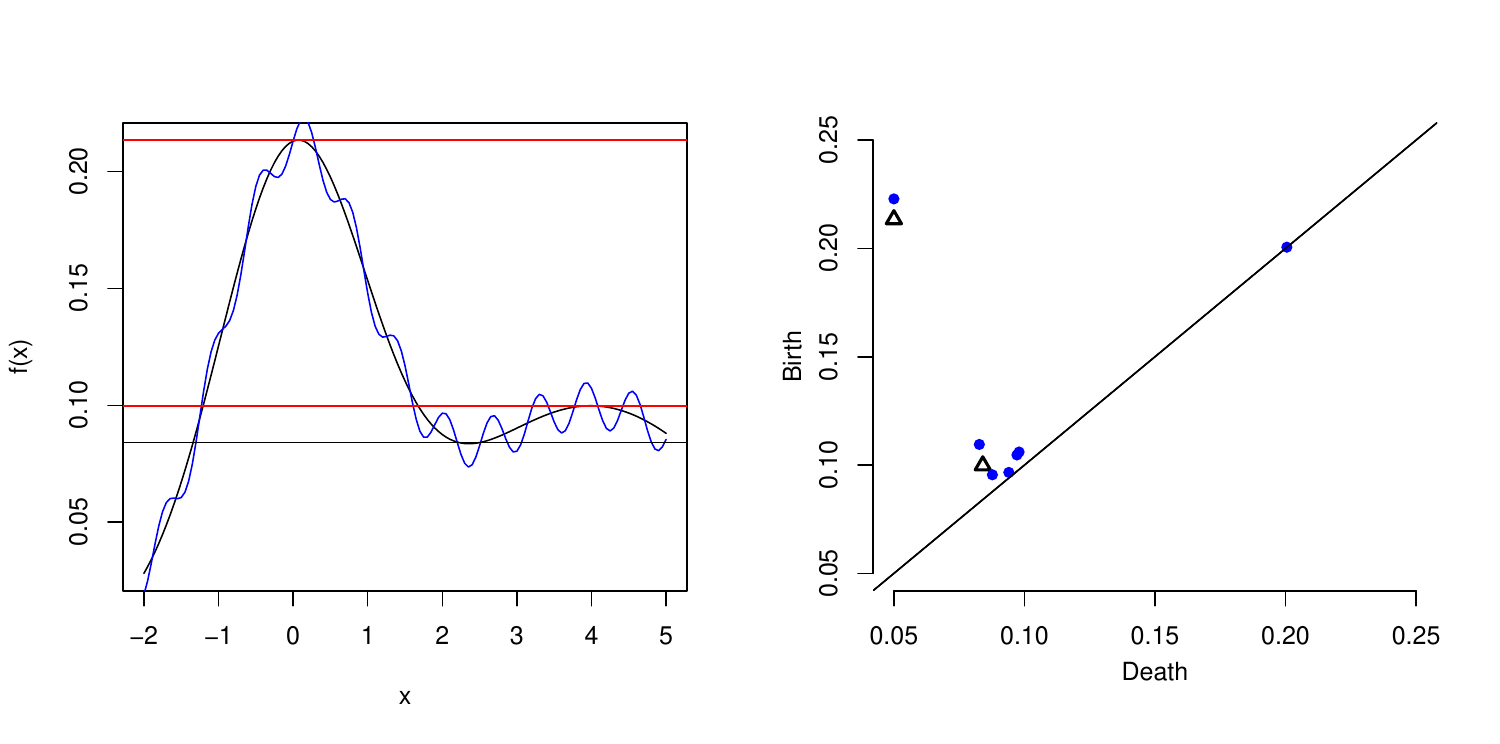}
 \caption{Persistence diagram $D\mathbb{P}$ (on the right) associated to the upper-star filtration of the functions whose graph is given on the left. The function in black is smooth, corresponding persistence diagram is given with black triangles. There are two elements in the persistence diagram, corresponding to the two peaks of the function. Peaks and valleys are highlighted on the left in red and black respectively, to read the values of birth and death in the persistent diagram. Notice that the high left point on the diagram is actually a point "at infinity", i.e. its death is $-\infty$. The function in blue is a noisy version with many peaks and valleys. Its persistence diagram corresponds to the blue points. There are two bluepoints a bit above the two black triangle, corresponding the the main topological features, and many points close to the diagonal, with low prominence, suggesting that this is topological noise. }
 
 \label{fig:DP}
\end{figure}
 
A \emph{persistence module} is a sequence of vector spaces $\mathbf{X}=(X_\alpha)_{\alpha\in\overline{\real}}$ where $\overline{\real}=\real\cup\{-\infty,+\infty\}$ together with linear maps $\varphi_{\beta,\alpha}\colon X_\beta\to X_\alpha$ whenever $\alpha \leq \beta$ (setting $X_\alpha\to X_\alpha$ as the identity) and such that, if $\alpha\leq \beta\leq \gamma$, then $\varphi_{\gamma,\alpha}=\varphi_{\beta,\alpha}\circ\varphi_{\gamma,\beta}$. In such a framework, one can study the \emph{persistence} of a vector. More precisely, given $\alpha\in \overline{\real}$ and $v\in X_\alpha$, we say that $v$ is \emph{born} at time\footnote{Here, according to the way our spaces are connected, the \emph{time} is flowing in the other direction: from $+\infty$ to $-\infty$.} $\alpha$ if $v$ is not in the image of  $\varphi_{\beta,\alpha}$ for all $\beta>\alpha$, and we say that it \emph{dies} at time $\gamma\leq \alpha$ if $\varphi_{\alpha,\gamma}(v)=0$ but $\varphi_{\alpha,\gamma'}(v)\neq 0$ for all $\gamma'$ with $\gamma<\gamma'<\alpha$.
Globally, we usually consider bases of the $X_\alpha's$ (and related to the linear maps $(\varphi_{\alpha,\beta})_{\beta\leq \alpha}$) and summarize their persistence with a persistence diagram. More precisely, the \emph{persistence diagram} $D \mathbf X$ of a persistence module $\mathbf X$ is the multi-set\footnote{A multi-set $A$ is a set with potential repetitions of elements, where we denote $\mu(p)$ the multiplicity of point $p \in \supp(A)$. It can be denoted $A = \bigcup_{p \in \lvert A \rvert} \coprod_{i=1}^{\mu(p)} p$ with $\supp(A)$ the support of $A$. } of points in $\overline{\real}^2$ consisting in the diagonal\footnote{The multiplicity of a point in the diagonal is $+\infty$.} $\Delta=\{(x,x)\mid x\in \overline{\real}\}$ and points $(i,j)$ for each basis element appearing at time $i$ and dying at time $j<i$. When reading a persistence diagram, one should consider the distance of the points to the diagonal, i.e., their \emph{prominence}. A point with low prominence should be considered as topological noise (they do not live long) whereas a point with high prominence as relevant topological information.
Persistence modules and diagrams are often used with homology, and we refer the reader to \cite{Hatcher00} for more details. Here we only use the 0-dimensional homology, which detects connected components. More precisely, if $T$ is a topological space or a graph, $H_0(T)$ is the vector space spanned by the (path) connected components of $T$. Moreover, a continuous map $T_1\to T_2$, between spaces or a graph homomorphism between graphs, induces a natural linear application $H_0(T_1)\to H_0(T_2)$. With that in hand, a classical example of persistence module is induced by the upper-star filtration of a function $\Prob\colon \INs \to \real^+$. If $\alpha\leq \beta$ are two reals, then there is an inclusion $\Prob^{-1}([\beta,+\infty])\subseteq \Prob^{-1}([\alpha,+\infty])$, and this induces linear maps $H_0(\Prob^{-1}([\beta,+\infty]))\to H_0(\Prob^{-1}([\alpha,+\infty]))$ which defines a persistence module. We denote by $D\Prob$ the associated persistence diagram. This notion is illustrated in Figure \ref{fig:DP} with two functions: the black one, very smooth with only two peaks, and the blue one, a noised version of the first one. 

\begin{figure}
 \includegraphics[scale=0.55, trim = 0 1cm 0 2cm, clip=TRUE]{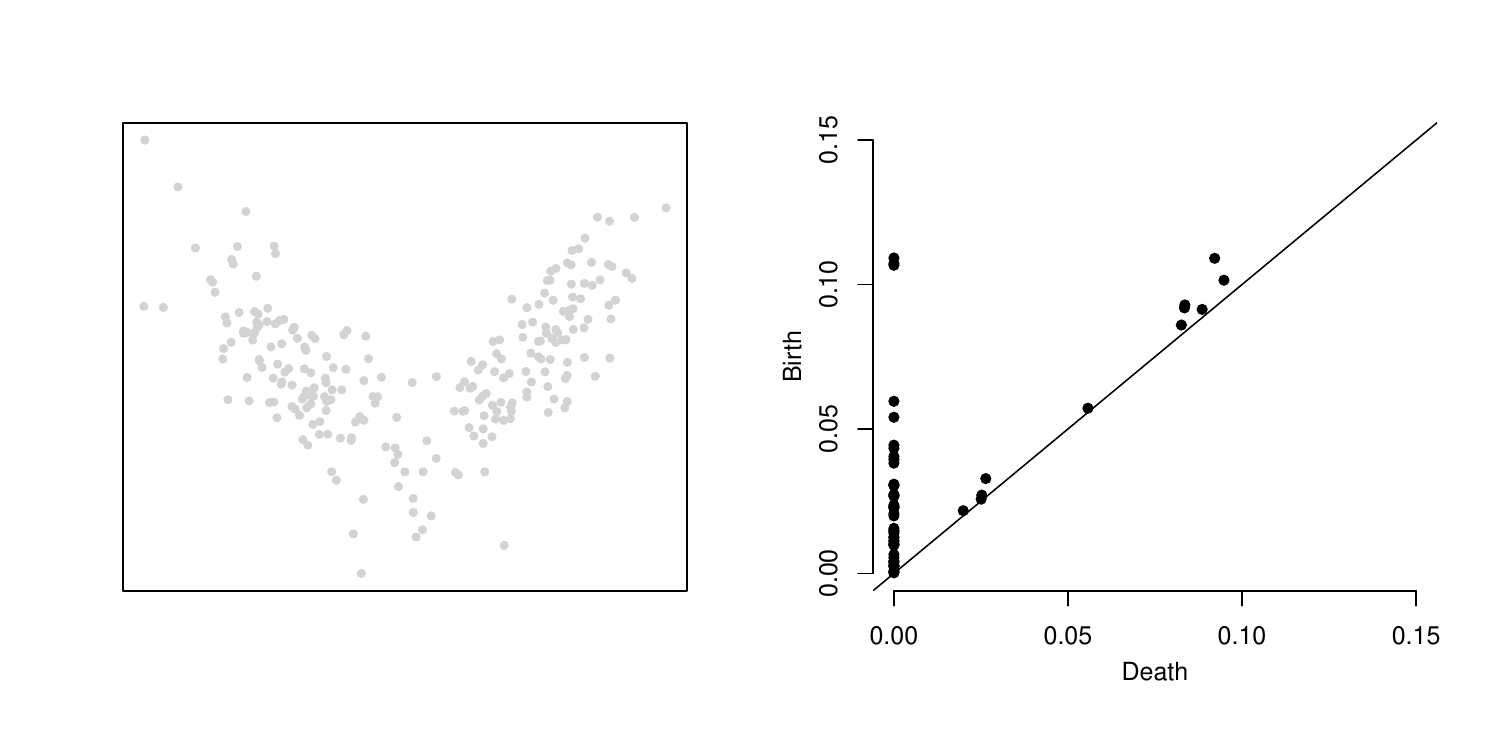}\\
 \includegraphics[scale=0.3, trim = 0 1cm 0 1cm, clip=TRUE]{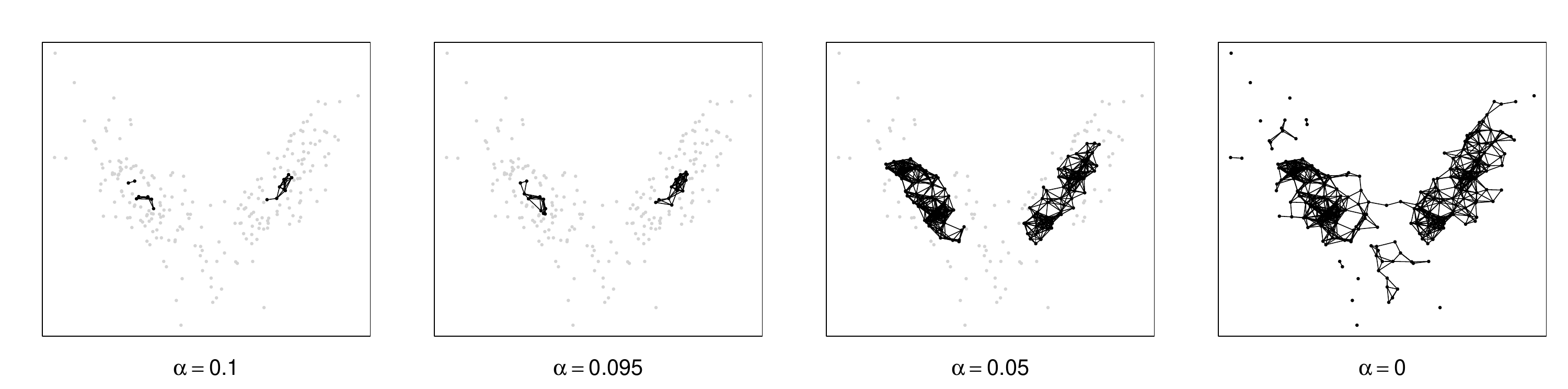}
 \caption{A point cloud generated from a mixture of two gaussians (top left), some realizations of its upper-star Rips filtration for $\delta=0.5$ and $\alpha \in \{0.1, 0.095, 0.05, 0\}$ (bottom, from left to right) and the associated persistent diagram (top right). }
 \label{fig:def3}
\end{figure}
 The persistence module that we consider here is the upper-star filtration of $\Prob\colon \INs\to \real$ restricted to a Rips graph. 
\begin{definition}[upper-star Rips filtration]
Given a finite point cloud $\Set_\obs$ from a metric space $(\INs, d)$ with a probability function $\Prob$ and a real value $\delta \in \real^+$, the \emph{upper-star Rips filtration} of $\Prob$, denoted $\mathcal{R}_\delta(\Set_\obs, \Prob)$, is the nested family of subgraphs of the Rips graph $R_\delta(\Set_\obs)$ defined as $\mathcal{R}_\delta(\Set_\obs, \Prob) = \left(R_\delta(\Set_\obs \cap \Prob^{-1}([\alpha, +\infty])\right)_{\alpha\in \overline{\real}}$.
Such a nested family of graphs gives rise to a persistence module 
\[\mathbf R_\delta(\Set_\obs, \Prob)=\left(H_0\left(R_\delta(\Set_\obs \cap \Prob^{-1}\left([\alpha, +\infty]\right))\right)\right)_{\alpha\in \overline{\real}}\]
and to its associated persistence diagram $D \mathbf R_\delta(\Set_\obs, \Prob)$. 
\end{definition}
This notion is illustrated in Figure \ref{fig:def3}. 

We define similarly the \emph{upper-star $\sigma$-Rips filtration} of $\Prob$, $\mathcal{R}_{\sigma(\cdot)}(\Set_\obs, \Prob)$ and the associated persistence module $\mathbf R_{\sigma(\cdot)}(\Set_\obs, \Prob)$ and persistent diagram $D \mathbf R_{\sigma(\cdot)}(\Set_\obs, \Prob)$.
In the next section, we give some tools to control the difference between $D \mathbf R_{\delta}(\Set_\obs, \Prob)$ and $D \mathbf R_{\sigma(\cdot)}(\Set_\obs, \Prob)$.

\subsection{Comparison of persistence diagrams for Rips graph and $\sigma$-Rips graph}
The bottleneck distance is an effective and natural proximity measure to compare two persistence diagrams.

\begin{definition}[bottleneck distance]
Given two multi-subsets $A_1,A_2$ of \;$\overline{\real}^2$ and a multi-bijection $\gamma: A_1 \rightarrow A_2$, the bottleneck distance $d^\infty_B(A_1, A_2)$ between $A_1$ and $A_2$ is the quantity:
\begin{align*}d^\infty_B(A_1, A_2) = \min_{\gamma:A_1 \rightarrow A_2} \max_{p \in A_1} \| p-\gamma(p) \|_\infty.
\end{align*}
\end{definition}

One can control the bottleneck distance between two persistence diagrams from upper star Rips filtration by comparing the evolution of the connected components along the filtration which can be track with the appearance level.

\begin{definition}[appearance level]
Given a finite point cloud $\Set_\obs = \{\obs_i\}_{i=1}^n$ from a metric space $(\INs, d)$ with a probability function $\Prob$ and $\delta$ such that $R_\delta(\Set_\obs)$ is connected. For two distinct points $(\obs_i, \obs_j)\in \Set_\obs^2$, we define the \emph{appearance level} $\alpha_\delta(\obs_i, \obs_j)$ as the highest level of the upper-star Rips filtration $\mathcal{R}_\delta(\Set_\obs, \Prob)$ at which $\obs_i$ and $\obs_j$ are in the same connected component:
\begin{equation*}
 \alpha_\delta(\obs_i, \obs_j) = 
 \max_{\gamma\in\mathcal{P}(\obs_i, \obs_j)}\min_{\obs\in \gamma}\Prob(\obs) 
\end{equation*}
where $\mathcal{P}(\obs_i, \obs_j)$ is the set of all paths\footnote{A path $\gamma$ in a graph $R$ is a sequence of vertices of $R$ where two consecutive vertices of $p$ are adjacent in $R$.} in $R_\delta(\Set_\obs)$ from the vertex $\obs_i$ to the vertex $\obs_j$.
\end{definition}
For example, in Figure \ref{fig:def3}, between $\alpha=0.1$ and $\alpha=0.095$, we see two connected components that are finally connected (on the cluster on the left). The appearance level of the corresponding points is a value between $0.095$ and $0.1$. 

We define similarly $\alpha_{\sigma(\cdot)}$ the appearance level for an upper-star $\sigma$-Rips filtration $\mathcal{R}_{\sigma(\cdot)}(\Set_\obs, \Prob)$.

Then, we are able to bound the bottleneck distance between the Rips graph and the $\sigma$-Rips graph.

\begin{theorem}
\label{main_theo}
Given a finite point cloud $\Set_\obs = \{\obs_i\}_{i=1}^n$ from a metric space $(\INs, d)$ with probability function $\Prob$. Let $R_{\delta}(\Set_\obs)$ be the Rips graph with parameter $\delta$, $R_{\sigma(\cdot)}(\Set_\obs)$ the $\sigma$-Rips graph with threshold function $\sigma$ and assume that they share the same connected components. Then, 

 \[ d_B^{\infty} \left(D\mathbf{R}_{\delta}(\Set_\obs,\Prob),D\mathbf{R}_{\sigma(\cdot)}(\Set_\obs,\Prob)\right) \leq \max_{(\obs_i,\obs_j)\in \Set_\obs^2 }\lvert \alpha_\delta(\obs_i,\obs_j)-\alpha_{\sigma(\cdot)}(\obs_i,\obs_j)\rvert,\]
 setting $\alpha_\delta(\obs_i,\obs_j)=\alpha_{\sigma(\cdot)}(\obs_i,\obs_j)=0$ if $\obs_i$ and $\obs_j$ are not in the same connected component. 
\end{theorem}

The proof stands in Appendix \ref{App:proof:thm}. It relies on \cite{Chazal09b} and their notion of $\epsilon$-interleaving. The main idea is that one can control the birth and the death of connected components along the filtration by controlling the changes in appearance levels.

This theorem means that, when switching from the metric distance $d$ to a closed (possibly) non-metric distance $\hat d$ defined in \eqref{dchapeau} (and then from the Rips graph to the $\sigma$-Rips graph), the dendrogram induced by the upper star Rips graph is mostly the same during the persistence process.

\section{Proper topological regions and their use in active learning}
\label{sec:method}

In this section, we start by defining the {proper topological regions}. We then use the proper topological regions for zero-shot learning and for pool-based active learning.

\subsection{Proper topological regions}
The main tool used in our method is the notion of topological regions that are based on the algorithm \tomato{} \citep{Chazal13}. 
\tomato\ is a clustering method that uses the hill climbing algorithm on a Rips graph $R_\delta(\Set_\obs)$ along with a merging rule on the Rips graph's persistence. It depends on a merging hyperparameter $\tau\geq 0$ which drives the granularity: it keeps only clusters with prominence higher than $\tau$. 
It can be easily adapted to work with a $\sigma$-Rips graph $R_{\sigma(\cdot)}(\Set_\obs)$ by considering the non-metric space $(\Set_{\mathbf{x}},\hat{d})$ with $\hat d$ introduced in Eq. \eqref{dchapeau}.
The topological regions correspond to the clusters given by \tomato\ for a $\sigma$-Rips graph, defined formally as follows. 
\begin{definition}
The \emph{topological regions} of a sample set $\Set_{\mathbf x}$ coming from an unknown marginal distribution $\Prob$ and with parameters $(\delta,r,t,\tau)$ are the clusters given by the clustering 
\[\TR_{\delta,r,t,\tau}^{\Set_\obs, \Prob} =\tomato_{\tau}\left(R_{\sigma(\cdot;\delta,r,t)}(\Set_\obs),\Prob\right).\]
When the set of covariates $\Set_\obs$, the underlying density $\Prob$, and the parameters are understood, we will simply denote $\TR$.
\end{definition}

For $\TR\colon \Set_\obs\to \{1,\ldots,k\}$ a clustering into $k$ topological regions of $\Set_\obs$, we denote $\mathcal L_{\TR}^\Prob$ the labeling function that propagates, in a given topological region, the label of the sample with the highest density with respect to $\mathbb{P}$:
\begin{align}
 \mathcal L^\Prob_{\TR} \colon \Set_\obs &\to \OUTs\nonumber\\
 \mathbf{x}_i &\mapsto \mathcal{O}\left(\underset{\obs_j : \TR(\obs_j) = \TR(\obs_i)}{\argmax}\;\Prob(\obs_j)\right).\label{eq:labelingFunction}
 \end{align}

If one could have access to the labeled data, we define the \emph{Purity Size function} $\PS$ as the objective function that considers the labeling error when propagating the labels inside the topological regions with $\mathcal{L}_{\TR}^\Prob$, penalized by the number of topological regions $k$ in $\TR$:
$$\PS\left(\Set,\Prob,\TR\right) = \left [ \frac{k}{n} + \frac{1}{n} \sum_{i = 1}^n 1_{\mathcal L_{\TR}^\Prob(\obs_i) \neq y_i} \right ] \in [0, 1].$$

Then, we introduce the notion of proper topological regions, that will be the key element in our method.

\begin{definition}
The \emph{proper topological regions} of a sample set $\Set_{\mathbf x}$ coming from an unknown marginal distribution $\Prob$ are the topological regions of $\TR^{\Set_\obs,\Prob}_{\delta^*, r^*, t^*, \tau^*}$ where
\begin{align}
 (\delta^*, r^*, t^*, \tau^*) =& \underset{(\delta,r,t, \tau)}{\argmin} \left\{ \PS\left(\Set,\Prob,\TR_{\delta,r,t,\tau}^{\Set_\obs, \Prob} \right) \right\}.\label{eq:opt}
 \end{align}
\end{definition}

However, in our active learning context, we need to use an unsupervised objective function. We consider a trade-off between the Silhouette score\footnote{Other potential unsupervised criteria typically used to assess the clustering quality are discussed in Appendix \ref{app:uns_crit}, but we have observed on an empirical study the benefit of the Silhouette.} and the coverage compactness of a clustering $\TR$ of $\Set_\obs$ into $k$ topological regions $\{R_1,\dots R_k\}$: for $1\leq q\leq k$, let $\pi_q$ be the cardinal of the topological region $R_q=\{\obs \in \Set_\obs: \TR(\obs) = q \}$. For $\lambda \in \mathbb{R}^+$, we define
\begin{align}
 \textit{SilSize}_\lambda(\Set_\obs, \TR) &= \left [\frac{1}{k} \sum_{q=1}^k \frac{1}{\pi_q} \sum_{\obs \in R_q} s_{il}(\obs)\right ] - \lambda \frac{k}{n}\in \left[-1-\lambda, 1-\frac{\lambda}{n}\right], \label{silsize}\\
 \text{with} \hspace{1cm} s_{il}(\obs) &= \frac{\nu^c(\obs) - \nu(\obs)}{\max(\nu(\obs), \nu^c(\obs))}\nonumber
 \end{align}
 where, for all $q$ and all $\obs \in R_q$, $\nu(\obs)$ is the average distance of sample $\obs$ within its cluster $R_q$ and $\nu^c(\obs)$ is the average distance of sample $\obs$ to his nearest neighbor cluster:
 \begin{align*}
 \nu(\obs) &= \frac{1}{\pi_q-1} \sum_{\obs' \in R_q} d(\obs, \obs'),
 &\nu^c(\obs) = \min_{q' \neq q} \frac{1}{\lvert C_{q'} \rvert} \sum_{\obs' \in \mathcal{C}_{q'}} d(\obs, \obs').
 \end{align*}

 Note that the trade-off parameter $\lambda$ in \eqref{silsize} is key in uncovering the proper topological regions of the sample set $\Set_\obs$. High values of $\lambda$ penalize the coverage compactness, resulting in partitions with a high degree of agglomeration, i.e. fewer topological regions with large cardinals. However, an additional way to control the labeling propagation error term of the Purity Size objective in an unsupervised setting is to control the size distribution of groups in the resulting partition. Conversely, lower $\lambda$ values result in highly fragmented partitions with many groups with small cardinals, and the Silhouette score tends to converge to graphs with a single non-singleton connected component and many singletons.

Thus, the optimization problem \eqref{eq:opt} is approximated by the following: 
\begin{align}
 % (\delta^*, r^*, t^*, \tau^*) =& 
 \underset{(\delta,r,t,\tau)}{\operatorname{argmax}} \left\{ \textit{SilSize}_\lambda\left(\Set_\obs, \TR_{\delta,r,t,\tau}^{\Set_\obs, \Prob}\right)\right\}.
\end{align}

Unfortunately, this optimization problem is too costly, because the set of parameters leads to running the \tomato\ function many times. So instead, we propose to solve the following proxy, which is running \tomato\ only once:
\begin{align}
 (\delta^\sharp, r^\sharp, t^\sharp) =& \underset{(\delta,r,t)}{\operatorname{argmax}} \left\{ \textit{SilSize}_\lambda(\Set_\obs,R_{\sigma(\cdot;\delta,r,t)}(\Set_\obs))\right\}\label{optim:sub1}\\
 \tau^\sharp=& \underset{\tau}{\operatorname{argmax}} \left\{ \textit{SilSize}_\lambda\left(\Set_\obs, \TR_{\delta^\sharp,r^\sharp,t^\sharp,\tau}^{\Set_\obs, \Prob}\right)\right\}\label{optim:sub2}
\end{align}
with a slight abuse of notations in \eqref{optim:sub1} between the Rips graph $R_{\sigma(\cdot;\delta,r,t)}(\Set_\obs)$ and its connected components seen as a clustering. 
The best hyperparameters $a^\sharp, r^\sharp, t^\sharp$ for the silhouette of the $\sigma$-Rips graph are then used to find the best hyperparameter $\tau^\sharp$ for the \tomato{} algorithm. 

\begin{algorithm}[t]
\begin{center}
\caption{Optimization procedure for PTR}\label{algo:opt}
\begin{algorithmic}[1]
\Require $\Set_\obs:= \{\obs_i\}_{i=1}^n$, $d: \mathcal{X} \times \mathcal{X} \to [0,\infty)$, $s$ the step size for the linear search, and $l$ the number of trials for the optimization strategy. 
\State Initialize $\lambda = s$.
\State Compute the density estimator $\hat{\Prob}$ with \eqref{eq:1} based on $d$ and $\Set_\obs$.
\State Optimize the problem~\eqref{optim:sub1} for $l$ trials, and return $(\hat \delta, \hat r, \hat t)$.
\State Build the $\sigma$-Rips graph $R_{\sigma(\cdot; \hat \delta, \hat r, \hat t)}(\Set_\obs)$.
\While{$R_{\sigma(\cdot; \hat \delta, \hat r, \hat t)}(\Set_\obs)$ is not a \textit{degenerate graph}\footnotemark
}
%\State Save the current graph parameters $ \hat a, \hat r, \hat t$. \label{algo:opt:delta} 
\State Update $\lambda \longleftarrow \lambda + s$.
\State Optimize the problem~\eqref{optim:sub1} for $l$ trials, updating $ \hat \delta, \hat r, \hat t$.
\State Build the $\sigma$-Rips graph $R_{\sigma(\cdot; \hat \delta, \hat r, \hat t)}(\Set_\obs)$.
\EndWhile
\State Update $\lambda \longleftarrow \lambda - s$.
\State Optimize problem~\eqref{optim:sub2} for $l$ trials
\State {\bfseries Output:} parameters $\hat \delta, \hat r, \hat t,\hat\tau$ and the corresponding $\widehat{\PTR}$.
\end{algorithmic}
\end{center}
\end{algorithm}
\footnotetext{A graph is degenerate if the sizes of the connected components are imbalanced (we do not allow very small connected components).}

Since the underlying density is usually unknown, we need to estimate it from the data. For that purpose, we use the distance to a measure \cite{Chazal13}, which computes the root-mean-squared distance to the $\ell$ nearest neighbors of the considered query point: for all $i \in \{1,\dots,n\}$, % is then calculated as follows:
\begin{align}
 \label{eq:1}
 \hat{\Prob}(\obs_i) = \left(\frac1\ell\sum_{j=1}^n d(\obs_i,\obs_j)^2 1_{\obs_j \text{ is a } \ell\text{-nearest neighbors of } \obs_i} \right)^{-1/2}.
\end{align}

The whole procedure used to approximate the proper topological regions is data-driven using the unlabeled set $\Set_\obs$, and we will denote in the following $\widehat\PTR$ the corresponding estimated proper topological regions with parameters $( \hat \delta, \hat r, \hat t, \hat \tau)$. 
We describe in Algorithm~\ref{algo:opt} a two-stage black-box optimization scheme to estimate the $\sigma$-Rips graph parameters $(\delta^*,r^*,t^*)$ by $( \hat \delta, \hat r, \hat t)$, and the merging parameter $\tau^*$ by $ \hat \tau$, solution to our optimization problem given in Eq.~\eqref{eq:opt}, for the proper topological regions of $\Set$. 

As we are extending \tomato\ to $\sigma$-Rips graph, proper topological regions enjoy the same theoretical guarantees as the topological region given by \tomato{} applied to the usual Rips graph.
More precisely, under some topological assumptions on the persistence diagrams, there is a range of values of $\tau$ such that the number of topological regions output by $\tomato_\tau(R_\delta(\mathcal{S}_\mathbf{x}),\Prob)$ is equal to the number of peaks (i.e., local maximum) of $\Prob$ with prominence at least $\tau$ in $D\Prob$ with high probability with respect to $n$. Moreover, each of these topological regions contain a neighborhood of the basins of attraction\footnote{The \emph{basin of attraction} of $\Prob\colon \mathcal X\to \real$ of a peak $p$ of $\Prob$ corresponds to all the points of $\mathcal X$ flowing into $p$ along the flow defined by the gradient vector field of $\Prob$.} of the corresponding peak. In this context, Theorem \ref{main_theo} tells us that, under reasonable conditions on the threshold function $\sigma$, we get about the same persistence diagram when considering the $\sigma$-Rips graph, and thus we derive the same kind of theoretical guarantees when applying \tomato\ with a $\sigma$-Rips graph.
Those results are summarized in Appendix \ref{app:tomato}.

\begin{algorithm}[t]
\begin{center}
\caption{Zero-shot learning based on proper topological regions}\label{algo:zeroshot}
\begin{algorithmic}[1]
\Require $\Set_\obs:= \{\obs_i\}_{i=1}^n$, oracle $\mathcal{O}$, budget $\mathcal{B}$, $\hat{\Prob}$, and proper topological regions $\widehat\PTR$. 
\State Detect the $\mathcal{B}$ largest proper topological regions $R_1,\ldots,R_\mathcal{B}$ of $\widehat\PTR$.
\State Set $\mathcal{S}_\obs^0 = \cup_{q=1}^{\mathcal{B}} R_q$ \label{algo:zeroshot:l2}
\For{$\mathbf{x}_i \in \mathcal{S}^0_\obs$}
\State Label the corresponding points using the oracle $\mathcal{O}$: $\hat y_i = \mathcal L_{\widehat\PTR}^{\hat{\Prob}}(\mathbf{x}_i)$.
\EndFor
\State {\bfseries Output:} $\hat{\mathcal{S}}^0 = (\mathbf x_i,\hat y_i)_{\mathbf{x}_i \in \mathcal{S}_\obs^0}$
\end{algorithmic}
\end{center}
\end{algorithm}

\begin{figure}
 \includegraphics[scale=0.4]{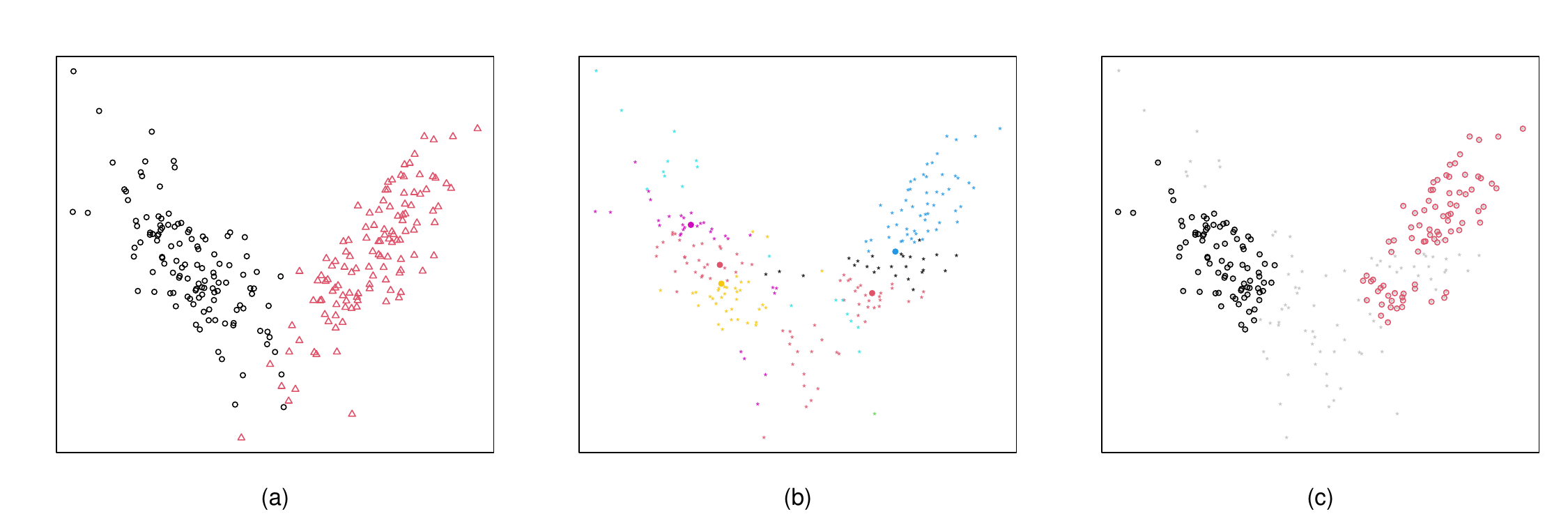}
 \caption{Illustration of Algorithm \ref{algo:zeroshot} (a) data with the oracle for the labels. (b) clustering given by tomato, describing the proper topological regions with the estimated parameters given by Algorithm \ref{algo:opt} (c) output of Algorithm \ref{algo:zeroshot}, where budget = 5, and with propagation. Not all points are labelled, and not all labels are sure. }
 \label{fig:algo2}
\end{figure}

\subsection{Proposed meta-learning strategy}
In this section, we introduce two strategies based on the proper topological regions found by Algorithm~\ref{algo:opt}. First, they are used in a zero-shot learning algorithm and second, in a pool-based strategy in a meta fashion, independently from the estimator and the strategy. To do so, we propose to use the label propagation scheme on proper topological regions in order to increase the sample size for training in a small budget scenario with a fixed number of calls to the oracle, as proposed in \cite{Fab18,citovsky2021batch}.

\paragraph{Zero-shot learning} 
We observe the unlabelled set $\Set_\obs$, and we have access to the proper topological regions $\widehat{\PTR}$ estimated by Algorithm~\ref{algo:opt}; to an oracle to give few labels (a budget $\mathcal{B}$ is considered); and the density estimation $\hat{\Prob}$. The strategy is the following: we label the $\mathcal{B}$ largest proper topological regions $R_1,R_2,\dots,R_{\mathcal{B}}$ using the labelling function $\mathcal{L}_{\widehat{\PTR}}^{\hat{\Prob}}$ defined in Eq. \eqref{eq:labelingFunction}: we ask to the oracle $\mathcal{B}$ points (the one with highest density in each $R_q$), and we then label $\sum_{q=1}^\mathcal{B} \lvert R_q \rvert$ points by propagating in each topological region. 
We denote by $\hat{\mathcal{S}}^0$ this first set of labeled points, which includes true labels obtained directly from the oracle, and estimated labels while diffusing the true labels to the topological regions. This procedure is summarized in Algorithm~\ref{algo:zeroshot} and illustrated in Figure \ref{fig:algo2}. The benefit to use proper topological regions instead of any clustering method is in details. No structure is assumed, as in k-means for example where clusters have a spherical shape. Here, only the topology is important, thus the algorithm can retrieve connected components even with an ambiguous shape. Moreover, fine hyperparameter tuning in \tomato{} allows to merge or distinguish between regions in a precise way (see Results \ref{R:1} and \ref{R:2} in Appendix \ref{app:tomato}).

\begin{algorithm}[t]
\begin{center}
\caption{Pool-based active learning on proper topological regions (PAL$_{\PTR}$)}\label{algo:main}
\begin{algorithmic}[1]
\Require $\Set_\obs:= \{\obs_i\}_{i=1}^n$, oracle $\mathcal{O}$, budget $\mathcal{B}$, $\hat{\Prob}$, proper topological regions $\widehat\PTR = \{R_1,\ldots,R_k\}$, active learner agent $h_{st}(\Set_\obs, \mathcal{B})$ with an underlying pool-based strategy $st$, and $r$ the active training rounds. 
%\State Compute the density estimator $\hat{\Prob}$ with \eqref{eq:1}.
\State Compute $\hat\Set^0$ using Algorithm 2
\For{$u = 0, \ldots, r-1$}
 \State Train the active learner agent $h_{st}(\hat\Set^u, \mathcal{B})$.
 \State Ask a set $S_\obs$ from $h_{st}$ of size $\mathcal{B}$.
\State $\Set^{u+1}_\obs = \cup_{R_q \cap S_\obs \neq \emptyset} R_q$
\If{$\tilde{\mathcal{B}} = \lvert\{q \mid R_q \cap S_\obs \neq \emptyset\}\rvert < \mathcal{B}$}
\State Detect the $\mathcal{B} - \tilde{\mathcal{B}}$ unlabeled largest topological regions without any points which do not intersect $S_\mathbf{x}$ and add them to $\Set^{u+1}_\obs$
%\State $\Set^{u+1}_\obs = \Set^{u+1}_\obs \cup (\cup_{ R_q})$
\EndIf
\For{$\mathbf{x}_i \in \mathcal{S}^{u+1}_\obs$}
\State Label the corresponding points using the oracle $\mathcal{O}$: $\hat y_i = \mathcal L_{\widehat\PTR}^{\hat{\Prob}}(\mathbf{x}_i)$.
\EndFor
\State $\hat{\Set}^{u+1} = \hat{\Set}^{u} \cup \left\{(\obs_i, \hat{y}_i)\mid{\obs_i \in \mathcal{S}^{u+1}_\obs}\right\}$
\EndFor
\State {\bfseries Output:} the labeled set $\hat{\Set}^{r}$
\end{algorithmic}
\end{center}
\end{algorithm}

\paragraph{Meta-approach for training pool-based active learning} 
The idea is again to diffuse the labels asked to the oracle to the proper topological regions to get more (pseudo)-labeled points. 
The unlabeled sample set $\Set_\obs$, the oracle $\mathcal{O}$, the budget $\mathcal{B}$, the density estimation $\hat{\Prob}$ and the proper topological regions $\widehat{\PTR} = \{R_1,\ldots, R_k\}$ estimated by Algorithm \ref{algo:opt} are common inputs for active learning techniques.
Additionally, a pool-based active learning technique $h_{st}(\Set_\obs, \mathcal{B})$ is also provided as input.
The algorithm then performs $r$ rounds of active learning. Within each round, the active learner agent is asked to detect $\mathcal{B}$ points, which are defining at most $\mathcal{B}$ topological regions. If two points detected by the active learner agent belong to the same topological region, we use the extra budget to label the largest topological regions without any detected points. Then, the oracle is asked to label $\mathcal{B}$ unlabeled examples that correspond to the points of high density in each considered topological region. The labeled set $\hat{\Set}^{r}$ returned by the algorithm has $r \times \mathcal{B}$ labels and many pseudo-labels, given by the label propagation with $\mathcal L_{\widehat\PTR}^{\hat{\Prob}}$ to increase the size of the training set during each round of active learning.

Remark that the choice for the extra budget to label the largest unlabelled proper topological regions is not driven by an active learner agent. Good results have been observed in Section \ref{sec:experiments}, but one can think of different ways to use the extra budget, such as, for example, asking the learner for more examples to label. Algorithm~\ref{algo:main} describes this meta-approach.

\begin{table}[t]
 \centering
 \caption{Dataset statistics: $n_{\text{train}}$ is the size of the training set, $n_{\text{test}}$ is the size of the test set, $m$ is the number of features, $c$ is the number of classes, and imbalance corresponds to the class imbalance ratio.}
 \label{tab:stats}
 \begin{tabular}{cccccc}
 Dataset & $n_{\text{train}}$& $n_{\text{test}}$ & $m$ & $c$ & imbalance \\ \hline
 protein \citep{protein}& $756$ & $324$& $77$ & $8$ & $0.70$ \\ \midrule
 banknote \citep{pmlb21}& $943$ & $405$ & $4$ & $2$ & $0.83$ \\ \midrule
 coil-20 \citep{coil20}& $1008$& $432$ & $1024$ & $20$ & $1.00$ \\ \midrule
 isolet \citep{isolet}& $4366$& $1872$ & $617$ & $26$ & $0.99$ \\ \midrule
 pendigits \citep{pmlb21}& $7694$& $3298$ & $16$ & $10$ & $0.92$ \\ \midrule
 nursery \citep{pmlb21}& $9070$ & $3888$& $8$ & $4$ & $0.09$ \\ %\midrule
 % adult & $34.2k$ & $14$ & $2$ & $0.31$ & $14.7k$ \\
 \bottomrule
 \end{tabular}
\end{table}

\section{Empirical results}
\label{sec:experiments}
We conduct a number of experiments aimed at evaluating how the proposed approach can identify valuable examples to be labeled for learning. To this end, we consider two scenarios for the identification of an initial training set to be labeled from an unlabeled set, and the increase of the training sample size during the rounds with active learning while operating under a low-budget regime.

We carry out experiments on data collections that are frequently used in active learning. Table \ref{tab:stats} presents statistics of these datasets.

For the metric function, $d$, we consider the Euclidean distance, and we choose the Tree-structured Parzen Estimator (TPE) \citep{Berg11} for the optimization procedure of Algorithm~\ref{algo:opt}, with a number of trials $l = 500$, and a step size $s$ of $0.01$ for the line search procedure. We estimate $\Prob$ using Eq. \eqref{eq:1} with the distance to measure based on the $\ell$ nearest neighbors with $\ell$ the sample size, if smaller than 2000, and 2000 elsewhere. In all our experiments, we use the random forest classifier \citep{ho95} as the base estimator for the different strategies with default parameters, we also consider several budgets $\mathcal{B} \in \{3, 10, 20\}$, and 20 stratified random splits, with $70\%$ of the data in the training set and $30\%$ in the test set. We report the balanced classification accuracy \citep{Bro10} over all experiments. Regarding the data preprocessing, we drop sample duplicates and samples with null values. Then we apply a standard min-max normalization to the filtered datasets.

\subsection{Rips graph vs $\sigma$-Rips graph} \label{sec:exp:graphs}
To validate our hypothesis of a density-aware threshold given by Eq. \eqref{sigma_delta_def} for class similarity and to motivate our generalization of the Rips graph to express this notion, we present a comparison study in Figure~\ref{fig:graphs_comp} between the Rips and the $\sigma$-Rips graphs on the protein dataset. The results for the other collections are shown in Figure~\ref{fig:all_datasets_graphs} in Appendix \ref{apdx:graphs_comp}.

\begin{figure}[t]
 \centering
 \caption{Comparison study between the Rips graph and the $\sigma$-Rips graph on the protein dataset: the Purity Size score is reported for each minimizer.}
 \label{fig:graphs_comp}
 \includegraphics[width=\textwidth, trim = 0 1.3cm 0 2cm, clip=TRUE]{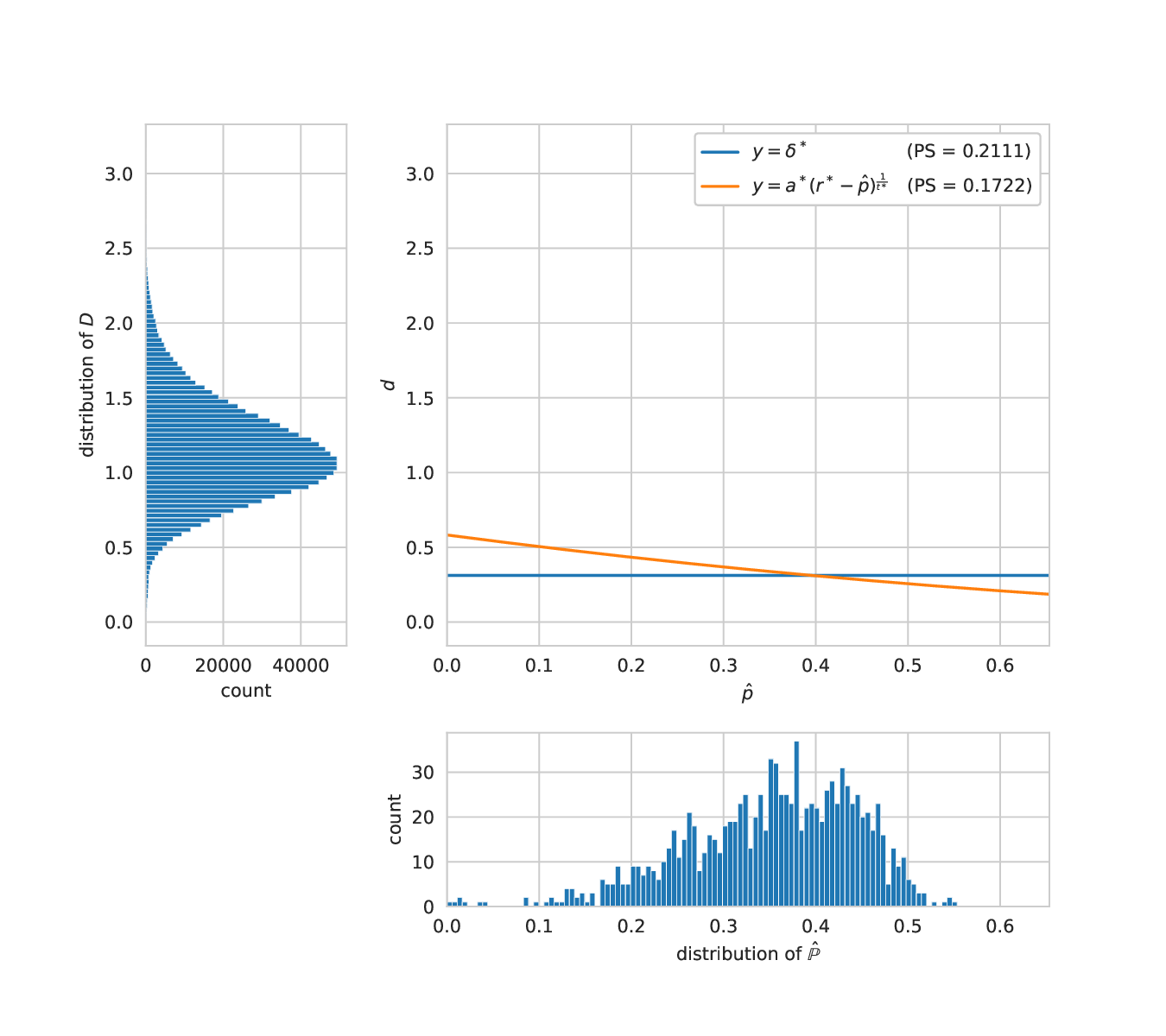}
\end{figure}

The plot represents the threshold of the best Rips graph and $\sigma$-Rips graph in minimizing the Purity Size cost function. The Rips graph's threshold (in blue) is a constant presented as a horizontal line in the plot. 

In this figure, we also include two more side plots which are the distribution of the dataset's density estimation $\hat{\Prob}$ under the x-axis and the distance matrix $D$'s distribution of Euclidean distances on the left of the y-axis. Note that from the definitions of the Rips graph (Def.~\ref{def:rips}), and the $\sigma$-Rips graph (Def.~\ref{def:sigma_rips}), threshold values larger than the maximum distance lead to a full graph.

From this figure, it comes out that the optimal threshold rule's values found in the hypothesis class of the $\sigma$-Rips graph with our proposed threshold function $\sigma(\cdot;( \delta,r,t,\tau))$ given in Eq. ~\eqref{sigma_delta_def} are negatively correlated to the estimation density $\hat{\Prob}$. We also observe that the best $\sigma$-Rips graph achieves better Purity Size than the best Rips graph. These observations are consistent with other datasets reported in the appendix, except for coil-20 and nursery collection, where they have the same performance. 
%Particularly for the coil-20 dataset where the optimal threshold seems to be a constant, we notice that, in this case, both graphs converge to the same threshold function, this shows that our proposed threshold function~\eqref{sigma_delta_def} can effectively approximate constant threshold functions. 
These findings provide empirical evidence for our hypothesis that class similarity is a density-aware measure. It also supports our choice of $\sigma(\cdot;(\delta,r,t,\tau))$ given in Eq.~\eqref{sigma_delta_def} as an appropriate threshold function to generalize the Rips graph.

\subsection{Cold-start results}
For the cold-start experiments, we consider the following unsupervised approaches to compare with our approach:
\begin{itemize}
 \item \textbf{K-Means clustering (KM).} The K-Means algorithm \citep{Lloyd82} partitions a collection of examples into $K$ clusters by minimizing the sum of squared distances to the cluster centers. It has been used for active learning in \citet{zhu08}, to generate the initial training set by labeling the closest sample to each centroid. 
 \item \textbf{K-Means clustering with model examples (KM+ME).} A variant of KM proposed in \citet{Kang04} adds artificial samples from the centroids, named \textit{model examples}, to the initial training set. This approach leads to an initial training set twice as large as the one created using K-Means.
 \item \textbf{K-Medoids clustering (Km).} The K-Medoids algorithm \citep{Kauf90} is very similar to K-Means except that it uses the actual samples for centers, namely the medoids, as the center of each cluster. These medoids are then used to form the initial training set in active learning.
 \item \textbf{Agglomerative Hierarchical Clustering (AHC).} Agglomerative hierarchical clustering \citep{Voor85} is a bottom-up clustering approach that builds a hierarchy of clusters. Initially, each sample represents a singleton cluster. Then, the algorithm recursively merges the closest clusters using a \textit{linkage function} (here the Ward linkage is used) until one cluster is left. This process is usually presented in a dendogram, where each level refers to a merge in the algorithm. AHC has been used for active learning in \citet{Dasg08} by pruning the dendogram at a certain level to obtain clusters, then similar to the strategy used with K-Means, selecting the closest samples to the clusters centroids to generate the initial training set. 
 \item \textbf{Furthest-First-Traversal (FFT).} The furthest-first traversal of a sample set is a sequence of a selected examples, where the first example is chosen arbitrarily, and each subsequent example in the chain is placed as far away from the set of previously chosen examples as possible. The resulting sequence is then used as the initial training set for active learning as in \citet{Baram04}.
 \item \textbf{Affinity Propagation Clustering (APC).} Affinity propagation is a clustering algorithm designed to find \textit{exemplars} of the sample set which are representative of clusters. It simultaneously considers all the sample set as possible \textit{exemplars} and uses the message-passing procedure to converge to a relevant set of \textit{exemplars}. The \textit{exemplars} found are then used as an initial training set for active learning \citep{Hu10}.
\end{itemize}
Our meta-approach for zero-shot learning is called PTR, where the $\sigma$-Rips graph is obtained by Algorithm \ref{algo:opt}. The results for Random Selection (RS) strategy, competitors, and our method PTR are shown in Table~\ref{tab:res1} over all collections for a budget $\mathcal{B} = 10$. Note that we do propagation within clusters detected by \tomato{} for our approach, but not for competitors, for which we consider $\mathcal{B}$ clusters. 

\begin{table}[t]
 \caption{Average balanced classification accuracy (in \%) and standard deviation of random forest classifier with the initial training set obtained from different methods over 20 stratified random splits for a budget $\mathcal{B} = 10$. $^\uparrow/^\downarrow$ indicate statistically significantly better/worse performance than Random Selection RS, according to a Wilcoxon rank sum test with $p < 0.05$ \citep{Lehman75}.}
 \label{tab:res1}
 \centering
 \begin{tabular}{llllllllll}
 \toprule
 Data & RS & KM & KM+ME & Km & AHC & FFT & APC & PTR\\
 \midrule
 {protein} & $28.2$ & $30.6^\uparrow$ & $31.4^\uparrow$ & $29.3^\uparrow$ & $31.6^\uparrow$ & $21.8$ & $28.8$ & $\mathbf{40.5}^\uparrow$ \\
 & $(3.2)$ & $(4.6)$ & $(4.5)$ & $(4.4)$ & $(3.7)$ & $(3.8)$ & $(3.4)$ & $(3.9)$ \\
 banknote & $79.9$ & $85.2^\uparrow$ & $86.8^\uparrow$ & $87.6^\uparrow$ & $85.6^\uparrow$ & $70.6^\downarrow$ & $82.4$ & $\mathbf{88.7}^\uparrow$ \\
 & $(9.9)$ & $(5.7)$ & $(4.8)$ & $(3.3)$ & $(5.0)$ & $(5.3)$ & $(6.9)$ & $(4.4)$ \\
 coil-20 & $29.0$ & $36.7^\uparrow$ & $38.2^\uparrow$ & $32.9^\uparrow$ & $36.0^\uparrow$ & $18.6^\downarrow$ & $27.2$ & $\mathbf{44.2}^\uparrow$\\
 & $(5.7)$ & $(4.2)$ & $(2.7)$ & $(5.1)$ & $(3.7)$ & $(3.4)$ & $(4.8)$ & $(2.4)$ \\
 isolet & $13.8$ & $22.3^\uparrow$ & $\mathbf{27.6}^\uparrow$ & $07.1^\downarrow$ & $23.3^\uparrow$ & $16.5^\uparrow$ & $15.4$ & $27.5^\uparrow$ \\
 & $(2.3)$ & $(1.6)$ & $(1.6)$ & $(1.9)$ & $(1.8)$ & $(1.7)$ & $(3.2)$ & $(2.8)$ \\
 pendigits & $37.4$ & $62.5^\uparrow$ & $65.6^\uparrow$ & $53.9^\uparrow$ & $61.4^\uparrow$ & $27.2^\downarrow$ & $38.3$ & $\mathbf{80.1}^\uparrow$ \\
 & $(7.2)$ & $(3.5)$ & $(2.3)$ & $(5.2)$ & $(1.9)$ & $(4.9)$ & $(8.2)$ & $(2.6)$ \\
 nursery & $42.7$ & $44.5$ & $\mathbf{49.3}^\uparrow$ & $28.4^\downarrow$ & $44.9$ & $39.1$ & $45.1$ & $46.5$ \\
 & $(7.2)$ & $(5.7)$ & $(4.0)$ & $(1.3)$ & $(7.2)$ & $(3.5)$ & $(6.7)$ & $(6.0)$ \\
 \bottomrule
 \end{tabular}
\end{table}

Except for the unbalanced dataset nursery, PTR consistently outperforms the random selection method, which has been demonstrated to be very difficult to surpass in different studies. This demonstrates that a preferable starting point for pool-based active learning procedures than random selection is to use the biggest proper topological regions discovered by Algorithm~\ref{algo:opt} as an initial training set (Line~\ref{algo:zeroshot:l2} of Algorithm~\ref{algo:zeroshot}). Furthermore, when compared to baseline approaches that are exclusively created to address the cold-start problem in active learning, our meta-approach exhibits very competitive results on different datasets. APC is equivalent to RS, while FFT and Km may have worst performance than RS in some settings. KM, KM+ME and AHL are better than RS, but we get the best results on four datasets out of the 6. We present further results for budgets $\mathcal{B}$ equal to 3 and 20 in Table~\ref{tab:res} in Appendix \ref{apdx:cold_start_comp}, with similar conclusions.

\subsection{Active learning results}
In this section, we present the results of our meta-approach for pool-based active learning strategies (described in Algorithm~\ref{algo:main} and denoted PAL$_{\PTR}$).

For the active learning experiments and following results from \citet{simoni19}, we only consider the comparison against the Random Selection strategy (RS) as a cold-start strategy, as it outperforms many recent strategies in active learning in small-budget scenarios. We compare our meta-approach and RS strategy for different pool-based active learning strategies, namely the uncertainty sampling query, the entropy sampling query, and the margin sampling query strategies \citep{modAL18}. 

Figure \ref{fig:all_10_uncertainty} depicts the results corresponding to the uncertainty sampling query strategy with budget $\mathcal B=10$ and a subfigure for each dataset made up of error bar plots that show the average balanced accuracy and standard deviation across the splits for all active learning rounds. Results for other strategies and budgets are shown in Figures~\ref{fig:protein_banknote}, \ref{fig:coil-20_isolet} and \ref{fig:pendigits_nursery} in Appendix \ref{Fig:poolBased}, the plots are organized so that each row and column represents a certain budget and active learning technique, but similar conclusions can be drawn. 

The results show that, as compared to the use of the random selection technique, all of the pool-based active learning strategies that were taken into consideration gain significantly from our method. Only in the nursery dataset do we not observe a gain. The primary cause is the nursery's significant class imbalance. In Algorithm~\ref{algo:main}, we decide to put the increase in training sample size ahead of class discovery or class ratio, which may help us understand the current class imbalance. These findings indicate that, when training with highly class-imbalanced datasets, various sample criteria of PTR in Algorithm~\ref{algo:main} should be taken into account in addition to only picking the largest ones.

\begin{figure}[t!]
 \centering
 \caption{Average balanced classification accuracy and standard deviation of pool-based active learning with the uncertainty strategy and budget $\mathcal B=10$ on protein dataset, using random forest estimator over 20 stratified random splits.}
 \label{fig:all_10_uncertainty}
 \begin{subfigure}[b]{0.49\textwidth}
 \centering
 \includegraphics[width=\textwidth]{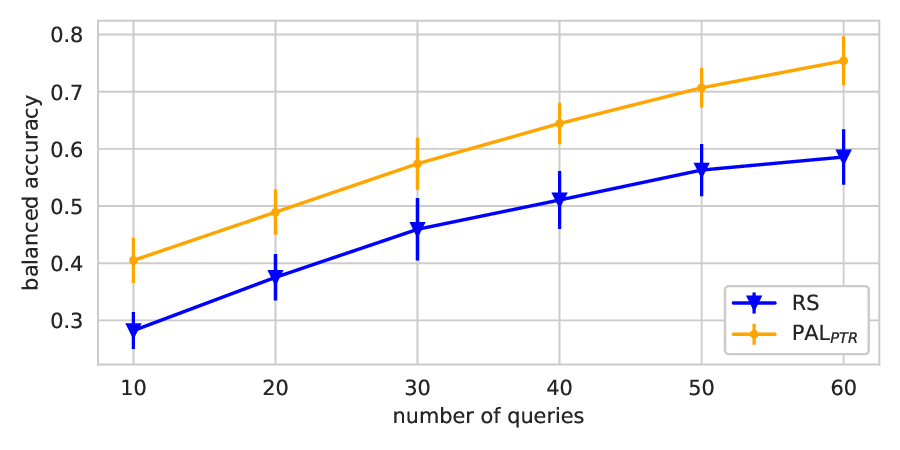}
 \caption{protein}
 \end{subfigure}
 \begin{subfigure}[b]{0.49\textwidth}
 \centering
 \includegraphics[width=\textwidth]{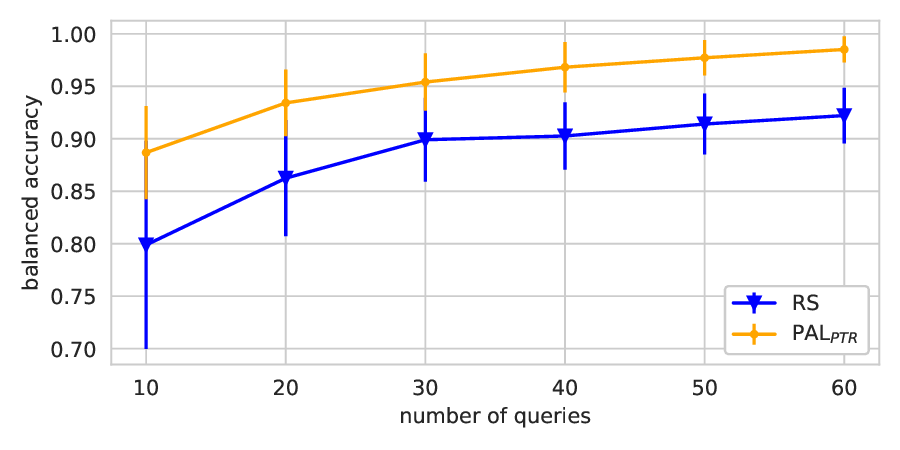}
 \caption{banknote}
 \end{subfigure}
 \begin{subfigure}[b]{0.49\textwidth}
 \centering
 \includegraphics[width=\textwidth]{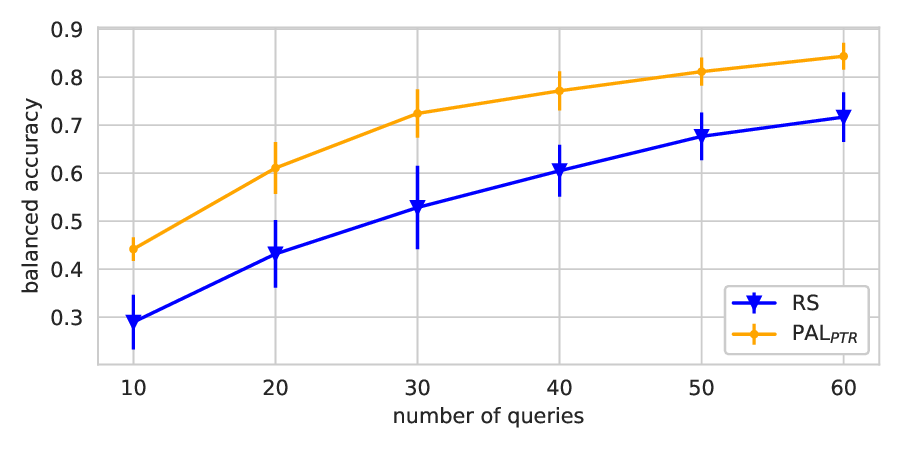}
 \caption{coil-20}
 \end{subfigure}
 \begin{subfigure}[b]{0.49\textwidth}
 \centering
 \includegraphics[width=\textwidth]{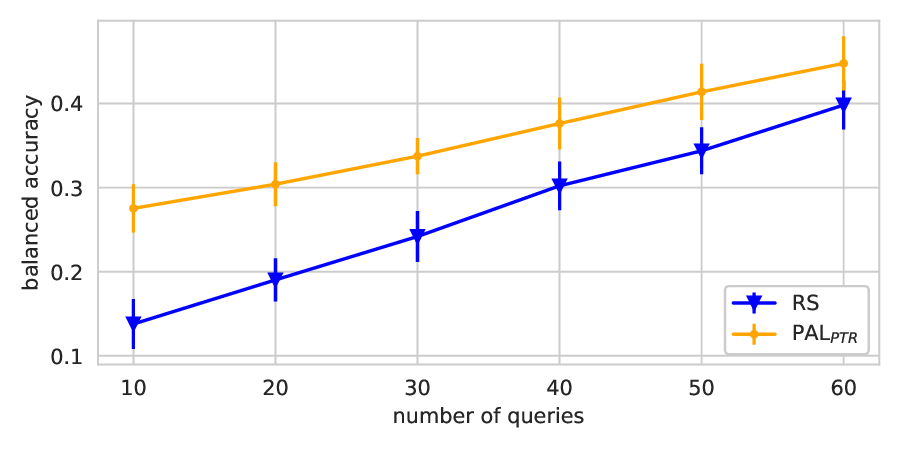}
 \caption{isolet}
 \end{subfigure}
 \begin{subfigure}[b]{0.49\textwidth}
 \centering
 \includegraphics[width=\textwidth]{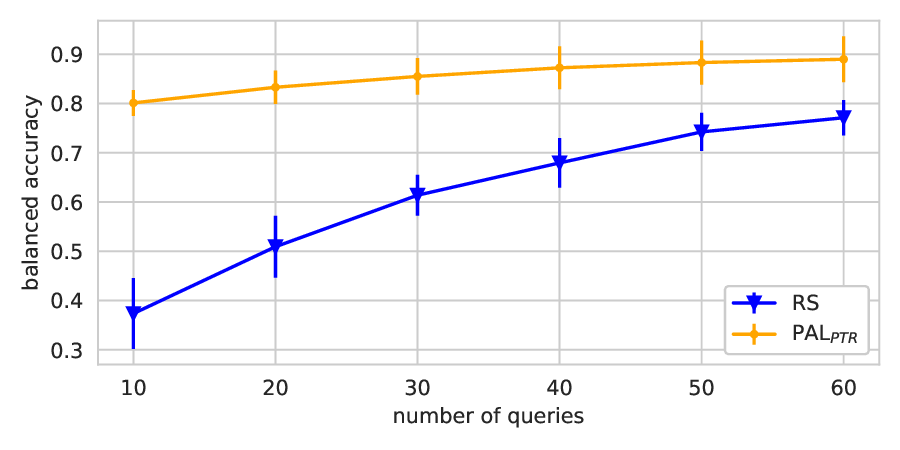}
 \caption{pendigits}
 \end{subfigure}
 \begin{subfigure}[b]{0.49\textwidth}
 \centering
 \includegraphics[width=\textwidth]{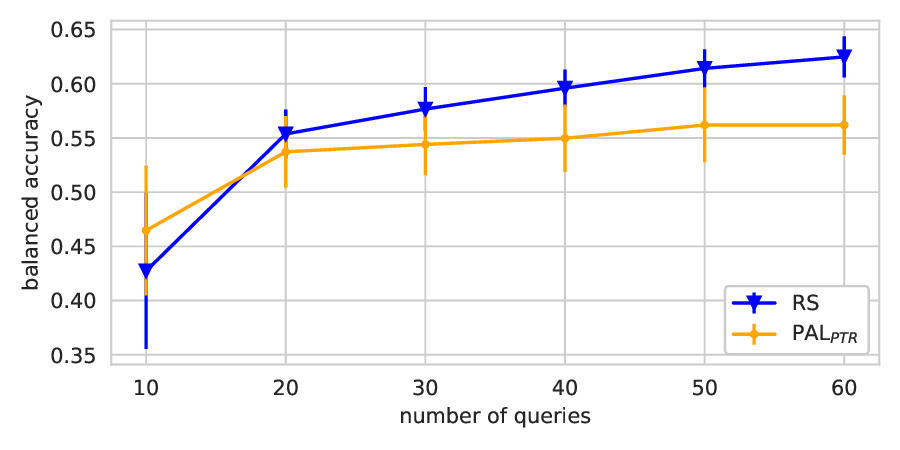}
 \caption{nursery}
 \end{subfigure}
\end{figure}

\section{Conclusion}\label{sec:concl}
We propose a data driven meta-approach for pool-based active learning strategies for multi-class classification problems. Our approach is based on the introduced notion of proper topological regions of a given sample set. We showed the theoretical foundations of this notion and derived a black-box optimization problem to uncover the proper topological regions. Then, we describe how to use those proper topological regions to select the first points to label in a zero-shot learning task, and we derive a meta-approach for pool-based active learning strategies. Our empirical study validates our meta-approach on different benchmarks, in low-budget scenarios, and for various pool-based active learning strategies. Challenging open questions are left: a theoretical analysis that guarantees good performance in active learning, such as generalization bounds, and the use of semi-supervised approaches to conclude the analysis with a model-dependent approach by having a regularization term derived from the PTR.

\bibliographystyle{apalike}
\bibliography{refs}

\begin{thebibliography}{}

\bibitem[Abe and Mamitsuka, 1998]{Abe98}
Abe, N. and Mamitsuka, H. (1998).
\newblock Query learning strategies using boosting and bagging.
\newblock In {\em Proceedings of the Fifteenth International Conference on
  Machine Learning}, ICML '98, page 1–9, San Francisco, CA, USA. Morgan
  Kaufmann Publishers Inc.

\bibitem[Ali et~al., 2014]{Ali14}
Ali, A., Caruana, R., and Kapoor, A. (2014).
\newblock Active learning with model selection.
\newblock In {\em AAAI}.

\bibitem[Amini and Usunier, 2015]{AminiUsunier15}
Amini, M.-R. and Usunier, N. (2015).
\newblock {\em Learning with Partially Labeled and Interdependent Data}.
\newblock Springer, New York, USA.

\bibitem[Andresini et~al., 2023]{andresini_2023}
Andresini, G., Appice, A., Ienco, D., and Malerba, D. (2023).
\newblock Seneca: Change detection in optical imagery using siamese networks
  with active-transfer learning.
\newblock {\em Expert Systems with Applications}, 214:119123.

\bibitem[Auer et~al., 2002]{Auer02}
Auer, P., Cesa-Bianchi, N., and Fischer, P. (2002).
\newblock Finite-time analysis of the multiarmed bandit problem.
\newblock {\em Machine Learning}, 47(2):235--256.

\bibitem[Baram et~al., 2004]{Baram04}
Baram, Y., El-Yaniv, R., and Luz, K. (2004).
\newblock Online choice of active learning algorithms.
\newblock {\em J. Mach. Learn. Res.}, 5:255–291.

\bibitem[Bergstra et~al., 2011]{Berg11}
Bergstra, J., Bardenet, R., Bengio, Y., and K\'{e}gl, B. (2011).
\newblock Algorithms for hyper-parameter optimization.
\newblock In Shawe-Taylor, J., Zemel, R., Bartlett, P., Pereira, F., and
  Weinberger, K., editors, {\em Advances in Neural Information Processing
  Systems}, volume~24, Red Hook, New York, USA. Curran Associates, Inc.

\bibitem[Berthelot et~al., 2019]{Bert19}
Berthelot, D., Carlini, N., Goodfellow, I., Papernot, N., Oliver, A., and
  Raffel, C.~A. (2019).
\newblock Mixmatch: A holistic approach to semi-supervised learning.
\newblock In Wallach, H., Larochelle, H., Beygelzimer, A., d~Alch\'{e}-Buc, F.,
  Fox, E., and Garnett, R., editors, {\em Advances in Neural Information
  Processing Systems}, volume~32, Red Hook, New York, USA. Curran Associates,
  Inc.

\bibitem[Bonnin et~al., 2011]{Bon11}
Bonnin, A., Borràs, R., and Vitrià, J. (2011).
\newblock A cluster-based strategy for active learning of rgb-d object
  detectors.
\newblock In {\em ICCV Workshops}, pages 1215--1220, New York, USA. IEEE.

\bibitem[Brodersen et~al., 2010]{Bro10}
Brodersen, K.~H., Ong, C.~S., Stephan, K.~E., and Buhmann, J.~M. (2010).
\newblock The balanced accuracy and its posterior distribution.
\newblock In {\em 2010 20th International Conference on Pattern Recognition},
  pages 3121--3124.

\bibitem[Carlsson, 2012]{carlsson12}
Carlsson, G. (2012).
\newblock {\em The Shape of Data}, page 16–44.
\newblock London Mathematical Society Lecture Note Series. Cambridge University
  Press, Cambridge, United Kingdom.

\bibitem[Carlsson and Gabrielsson, 2020]{Gar20}
Carlsson, G. and Gabrielsson, R.~B. (2020).
\newblock Topological approaches to deep learning.
\newblock In Baas, N.~A., Carlsson, G.~E., Quick, G., Szymik, M., and Thaule,
  M., editors, {\em Topological Data Analysis}, pages 119--146, Cham. Springer
  International Publishing.

\bibitem[Chapelle et~al., 2006]{Chapman06}
Chapelle, O., Sch{\"o}lkopf, B., and Zien, A. (2006).
\newblock {\em Semi-Supervised Learning}.
\newblock MIT Press, Cambridge, Massachusetts, USA.

\bibitem[Chazal et~al., 2009]{Chazal09b}
Chazal, F., Cohen-Steiner, D., Glisse, M., Guibas, L.~J., and Oudot, S.~Y.
  (2009).
\newblock Proximity of persistence modules and their diagrams.
\newblock In {\em Proceedings of the Twenty-Fifth Annual Symposium on
  Computational Geometry}, SCG '09, page 237–246, New York, NY, USA.
  Association for Computing Machinery.

\bibitem[Chazal et~al., 2014]{Chazal14}
Chazal, F., de~Vin~Silva, and Oudot, S. (2014).
\newblock Persistence stability for geometric complexes.
\newblock {\em Geometriae Dedicata}, 173(1):193--214.

\bibitem[Chazal et~al., 2011]{Chazal2011}
Chazal, F., Guibas, L.~J., Oudot, S.~Y., and Skraba, P. (2011).
\newblock Scalar field analysis over point cloud data.
\newblock {\em Discrete \& Computational Geometry}, 46(4):743.

\bibitem[Chazal et~al., 2013]{Chazal13}
Chazal, F., Guibas, L.~J., Oudot, S.~Y., and Skraba, P. (2013).
\newblock Persistence-based clustering in riemannian manifolds.
\newblock {\em J. ACM}, 60(6).

\bibitem[Chen et~al., 2022]{Chen22}
Chen, L., Bai, Y., Huang, S., Lu, Y., Wen, B., Yuille, A.~L., and Zhou, Z.
  (2022).
\newblock Making your first choice: To address cold start problem in vision
  active learning.
\newblock {\em ArXiv}, abs/2210.02442.

\bibitem[Citovsky et~al., 2021]{citovsky2021batch}
Citovsky, G., DeSalvo, G., Gentile, C., Karydas, L., Rajagopalan, A.,
  Rostamizadeh, A., and Kumar, S. (2021).
\newblock Batch active learning at scale.
\newblock In Beygelzimer, A., Dauphin, Y., Liang, P., and Vaughan, J.~W.,
  editors, {\em Advances in Neural Information Processing Systems}.

\bibitem[Danka and Horvath, 2018]{modAL18}
Danka, T. and Horvath, P. (2018).
\newblock mod{AL}: {A} modular active learning framework for {P}ython.
\newblock available on arXiv at \url{https://arxiv.org/abs/1805.00979}.

\bibitem[Dasgupta and Hsu, 2008]{Dasg08}
Dasgupta, S. and Hsu, D. (2008).
\newblock Hierarchical sampling for active learning.
\newblock In {\em Proceedings of the 25th International Conference on Machine
  Learning}, ICML '08, page 208–215, New York, NY, USA. Association for
  Computing Machinery.

\bibitem[Edelsbrunner and Harer, 2010]{Edel10}
Edelsbrunner, H. and Harer, J. (2010).
\newblock {\em Computational Topology - an Introduction}.
\newblock American Mathematical Society, Boston, USA.

\bibitem[Fanty and Cole, 1991]{isolet}
Fanty, M. and Cole, R. (1991).
\newblock Spoken letter recognition.
\newblock In Lippmann, R.~P., Moody, J.~E., and Touretzky, D.~S., editors, {\em
  Advances in Neural Information Processing Systems 3}, pages 220--226.
  Morgan-Kaufmann, Massachusetts, USA.

\bibitem[Garnett, 2022]{garnett22}
Garnett, R. (2022).
\newblock {\em Bayesian Optimization}.
\newblock Cambridge University Press, Cambridge, United Kingdoms.

\bibitem[Guyon et~al., 2011]{guyon11}
Guyon, I., Cawley, G.~C., Dror, G., and Lemaire, V. (2011).
\newblock Results of the active learning challenge.
\newblock In Guyon, I., Cawley, G., Dror, G., Lemaire, V., and Statnikov, A.,
  editors, {\em Active Learning and Experimental Design workshop In conjunction
  with AISTATS 2010}, volume~16 of {\em Proceedings of Machine Learning
  Research}, pages 19--45, Sardinia, Italy. PMLR.

\bibitem[Hatcher, 2000]{Hatcher00}
Hatcher, A. (2000).
\newblock {\em {Algebraic topology}}.
\newblock Cambridge Univ. Press, Cambridge.

\bibitem[Hausmann, 1995]{Haus95}
Hausmann, J.-C. (1995).
\newblock {\em On the Vietoris-Rips complexes and a cohomology theory for
  metric spaces}, pages 175--188.
\newblock Prospects in topology : proceedings of a conference in honor of
  William Browder. Princeton University Press, Princeton, N.J.
\newblock ID: unige:12821.

\bibitem[Higuera et~al., 2015]{protein}
Higuera, C., Gardiner, K.~J., and Cios, K.~J. (2015).
\newblock Self-organizing feature maps identify proteins critical to learning
  in a mouse model of down syndrome.
\newblock {\em PLOS ONE}, 10(6):1--28.

\bibitem[Ho, 1995]{ho95}
Ho, T.~K. (1995).
\newblock Random decision forests.
\newblock In {\em Proceedings of 3rd international conference on document
  analysis and recognition}, volume~1, pages 278--282. IEEE.

\bibitem[Hu et~al., 2010]{Hu10}
Hu, R., Namee, B.~M., and Delany, S.~J. (2010).
\newblock Off to a good start: Using clustering to select the initial training
  set in active learning.
\newblock In {\em FLAIRS}.

\bibitem[Jiang et~al., 2021]{Jia21}
Jiang, Y., Chen, D., Chen, X., Li, T., Wei, G.-W., and Pan, F. (2021).
\newblock Topological representations of crystalline compounds for the
  machine-learning prediction of materials properties.
\newblock {\em npj Computational Materials}, 7(1):28.

\bibitem[Kang et~al., 2004]{Kang04}
Kang, J., Ryu, K.~R., and chul Kwon, H. (2004).
\newblock Using cluster-based sampling to select initial training set for
  active learning in text classification.
\newblock In {\em PAKDD}.

\bibitem[Kaufman and Rousseeuw, 1990]{Kauf90}
Kaufman, L. and Rousseeuw, P. (1990).
\newblock {\em Finding Groups in Data: An Introduction To Cluster Analysis}.

\bibitem[Krempl et~al., 2015]{Krem15}
Krempl, G., Ha, T.~C., and Spiliopoulou, M. (2015).
\newblock Clustering-based optimised probabilistic active learning (copal).
\newblock In Japkowicz, N. and Matwin, S., editors, {\em Discovery Science},
  pages 101--115, Cham. Springer International Publishing.

\bibitem[Krishnapriyan et~al., 2021]{Kri21}
Krishnapriyan, A.~S., Montoya, J., Haranczyk, M., Hummelsh{\o}j, J., and
  Morozov, D. (2021).
\newblock Machine learning with persistent homology and chemical word
  embeddings improves prediction accuracy and interpretability in metal-organic
  frameworks.
\newblock {\em Scientific Reports}, 11(1):8888.

\bibitem[Lakshminarayanan et~al., 2017]{lakshmi2016}
Lakshminarayanan, B., Pritzel, A., and Blundell, C. (2017).
\newblock Simple and scalable predictive uncertainty estimation using deep
  ensembles.
\newblock In {\em Advances in Neural Information Processing Systems 30}, pages
  6402--6413.

\bibitem[Lewis and Catlett, 1994]{Lew94}
Lewis, D.~D. and Catlett, J. (1994).
\newblock Heterogeneous uncertainty sampling for supervised learning.
\newblock In Cohen, W.~W. and Hirsh, H., editors, {\em Machine Learning
  Proceedings 1994}, pages 148--156. Morgan Kaufmann, San Francisco (CA).

\bibitem[Li et~al., 2020]{Li20}
Li, W., Dasarathy, G., Natesan~Ramamurthy, K., and Berisha, V. (2020).
\newblock Finding the homology of decision boundaries with active learning.
\newblock In Larochelle, H., Ranzato, M., Hadsell, R., Balcan, M., and Lin, H.,
  editors, {\em Advances in Neural Information Processing Systems}, volume~33,
  pages 8355--8365. Curran Associates, Inc.

\bibitem[Lloyd, 1982]{Lloyd82}
Lloyd, S. (1982).
\newblock Least squares quantization in pcm.
\newblock {\em IEEE Transactions on Information Theory}, 28(2):129--137.

\bibitem[Lughofer, 2012]{Lughofer12}
Lughofer, E. (2012).
\newblock Single-pass active learning with conflict and ignorance.
\newblock {\em Evolving Systems}, 3(4):251--271.

\bibitem[Lum et~al., 2013]{Lum13}
Lum, P., Singh, G., Lehman, A., Ishkanov, T., Vejdemo-Johansson, M., Alagappan,
  M., Carlsson, J., and Carlsson, G. (2013).
\newblock Extracting insights from the shape of complex data using topology.
\newblock {\em Scientific reports}, 3:1236.

\bibitem[Perez et~al., 2018]{Fab18}
Perez, F., Lebret, R., and Aberer, K. (2018).
\newblock Cluster-based active learning.
\newblock {\em CoRR}, page abs/1812.11780.

\bibitem[Pourahmadi et~al., 2021]{pour21}
Pourahmadi, K., Nooralinejad, P., and Pirsiavash, H. (2021).
\newblock A simple baseline for low-budget active learning.
\newblock {\em arXiv preprint arXiv:2110.12033}.

\bibitem[Rieck et~al., 2020]{Rieck20}
Rieck, B., Yates, T., Bock, C., Borgwardt, K., Wolf, G., Turk-Browne, N., and
  Krishnaswamy, S. (2020).
\newblock Uncovering the topology of time-varying fmri data using cubical
  persistence.
\newblock In Larochelle, H., Ranzato, M., Hadsell, R., Balcan, M.~F., and Lin,
  H., editors, {\em Advances in Neural Information Processing Systems},
  volume~33, pages 6900--6912, Red Hook, New York, USA. Curran Associates, Inc.

\bibitem[Romano et~al., 2021]{pmlb21}
Romano, J.~D., Le, T.~T., La~Cava, W., Gregg, J.~T., Goldberg, D.~J.,
  Chakraborty, P., Ray, N.~L., Himmelstein, D., Fu, W., and Moore, J.~H.
  (2021).
\newblock Pmlb v1.0: an open source dataset collection for benchmarking machine
  learning methods.
\newblock {\em arXiv preprint arXiv:2012.00058v2}.

\bibitem[Roy and McCallum, 2001]{RoyM01}
Roy, N. and McCallum, A. (2001).
\newblock Toward optimal active learning through sampling estimation of error
  reduction.
\newblock In Brodley, C.~E. and Danyluk, A.~P., editors, {\em Proceedings of
  the Eighteenth International Conference on Machine Learning {(ICML} 2001),
  Williams College, Williamstown, MA, USA, June 28 - July 1, 2001}, pages
  441--448, Massachusetts, USA. Morgan Kaufmann.

\bibitem[Settles, 2009]{settles09}
Settles, B. (2009).
\newblock Active learning literature survey.
\newblock Computer Sciences Technical Report 1648, University of
  Wisconsin--Madison.

\bibitem[Siméoni et~al., 2019]{simoni19}
Siméoni, O., Budnik, M., Avrithis, Y., and Gravier, G. (2019).
\newblock Rethinking deep active learning: Using unlabeled data at model
  training.
\newblock {\em ICPR}.

\bibitem[Singh et~al., 2007]{Sin07}
Singh, G., Mémoli, F., and Carlsson, G. (2007).
\newblock Topological methods for the analysis of high dimensional data sets
  and 3d object recognition.
\newblock pages 91--100.

\bibitem[Thoreau et~al., 2022]{thoreau_2022}
Thoreau, R., Achard, V., Risser, L., Berthelot, B., and Briottet, X. (2022).
\newblock Active learning on large hyperspectral datasets: A preprocessing
  method.
\newblock {\em The International Archives of the Photogrammetry, Remote Sensing
  and Spatial Information Sciences}, XLIII-B3-2022:435--442.

\bibitem[Urner et~al., 2013]{Urner13}
Urner, R., Wulff, S., and Ben-David, S. (2013).
\newblock Plal: Cluster-based active learning.
\newblock In Shalev-Shwartz, S. and Steinwart, I., editors, {\em Proceedings of
  the 26th Annual Conference on Learning Theory}, volume~30 of {\em Proceedings
  of Machine Learning Research}, pages 376--397, Princeton, NJ, USA. PMLR.

\bibitem[Voorhees, 1985]{Voor85}
Voorhees, E.~M. (1985).
\newblock {\em The Effectiveness {\&} Efficiency of Agglomerative Hierarchic
  Clustering in Document Retrieval}.
\newblock PhD thesis, Cornell University, {USA}.

\bibitem[Wolfe, 2012]{Lehman75}
Wolfe, D.~A. (2012).
\newblock {\em Nonparametrics: Statistical Methods Based on Ranks and Its
  Impact on the Field of Nonparametric Statistics}, pages 1101--1110.
\newblock Springer US, Boston, MA.

\bibitem[Yan et~al., 2011]{yan11}
Yan, Y., Rosales, R., Fung, G., and Dy, J.~G. (2011).
\newblock Active learning from crowds.
\newblock In {\em Proceedings of the 28th International Conference on
  International Conference on Machine Learning}, page 1161–1168.

\bibitem[Yang et~al., 2011]{coil20}
Yang, J., Chen, Z., Chen, W.-S., and Chen, Y. (2011).
\newblock Robust affine invariant descriptors.
\newblock {\em Mathematical Problems in Engineering}.

\bibitem[Yu and Hansen, 2017]{Yu17}
Yu, C. and Hansen, J. H.~L. (2017).
\newblock Active learning based constrained clustering for speaker diarization.
\newblock {\em IEEE/ACM Transactions on Audio, Speech, and Language
  Processing}, 25(11):2188--2198.

\bibitem[Zhu et~al., 2008]{zhu08}
Zhu, J., Wang, H., Yao, T., and Tsou, B.~K. (2008).
\newblock Active learning with sampling by uncertainty and density for word
  sense disambiguation and text classification.
\newblock In {\em Proceedings of the 22nd International Conference on
  Computational Linguistics (Coling 2008)}, pages 1137--1144, Manchester, UK.

\end{thebibliography}

\newpage

\appendix
\section{Proof of Theorem \ref{main_theo}}\label{App:proof:thm}

To show that two persistence diagrams are close to one another with respect to the bottleneck distance, one can use the following notion introduced in \cite{Chazal09b}.

\begin{definition}[$\varepsilon$-interleaved]
\label{def:interleaved}
Let $\mathbf{X}=(X_\alpha)_{\alpha\in\overline{\real}}$ and $\mathbf{Y}=(Y_\alpha)_{\alpha\in\overline{\real}}$ be two persistence modules and let $D\mathbf{X}$. We say that $\mathbf{X}$ and $\mathbf{Y}$ are \emph{strongly $\varepsilon$-interleaved} if there exist two families of linear application $\{\varphi_\alpha\colon X_\alpha\to Y_{\alpha-\varepsilon}\}_{\alpha\in{\real}}$ and $\{\psi_\alpha\colon Y_\alpha\to X_{\alpha-\varepsilon}\}_{\alpha\in {\real}}$, such that for all $\alpha,\beta\in\mathbb R$, if $\alpha\leq \beta$, then the following diagrams, whenever they make sense, are commutative:

\resizebox{\linewidth}{!}{
 \begin{minipage}{\linewidth}
 \begin{align*}
 \xymatrix @!0 @R=2cm @C=1cm {
 &X_{\beta+\varepsilon}\ar[rrr]\ar_{\varphi_{\beta+\varepsilon}}[dl] & & & X_{\alpha-\varepsilon} \\
 Y_{\beta}\ar[rrrrr] & & & & & Y_{\alpha}\ar_{\psi_{\alpha}}[ul]&} 
 &&
 \xymatrix @!0 @R=2cm @C=1cm {
 &X_{\beta}\ar[rrrrr]\ar_{\varphi_{\beta}}[dl] & & & & & X_\alpha\ar_{\varphi_{\alpha}}[dl] \\
 Y_{\beta-\varepsilon}\ar[rrrrr] & & & & & Y_{\alpha-\varepsilon}} 
 \\&&\\
 \xymatrix @!0 @R=2cm @C=1cm {
 X_{\beta}\ar[rrrrr] & & & & & X_{\alpha}\ar_{\varphi_{\alpha}}[dl]& \\
 & Y_{\beta+\varepsilon}\ar_{\psi_{\beta+\varepsilon}}[ul]\ar[rrr] & & & Y_{\alpha-\varepsilon}} 
 &&
 \xymatrix @!0 @R=2cm @C=1cm {
 X_{\beta-\varepsilon}\ar[rrrrr] & & & & & X_{{\alpha}-\varepsilon} \\
 &Y_{\beta}\ar[rrrrr]\ar_{\psi_{\beta}}[ul] & & & & & Y_{\alpha}\ar_{\psi_{\alpha}}[ul]}
 \end{align*}
 \end{minipage}
}
\end{definition}

The idea behind these diagrams is that every component appearing (resp. dying) in $\mathbf{X}$ at some time $\alpha$ must appear (resp. die) in $\mathbf{Y}$ within $[\alpha - \varepsilon , \alpha + \varepsilon]$, and vice-versa. The following lemma highlights how important this notion is.

\begin{lemma}\label{lem:interleaved}
Let $\mathbf{X}$ and $\mathbf{Y}$ be two persistence modules such that $D\mathbf{X}$ and $D\mathbf{Y}$ have only finitely many points away from the diagonal, and let $\varepsilon>0$. If $\mathbf{X}$ and $\mathbf{Y}$ are strongly $\varepsilon$-interleaved, then $D\mathbf{X}$ and $D\mathbf{Y}$ are at a distance at most $\varepsilon$ with respect to the bottleneck distance. 
\end{lemma}

This lemma is a direct consequence of \citet[Theorem 4.4]{Chazal09b} where the result is proven for every homological dimension.

For example, in \citet[Theorem 5]{Chazal2011}, it is proven that given the density function $\Prob$ on a point cloud $\Set_\obs$ with sufficient sampling density, the persistence diagram $D\real_\delta(\Set_\obs, \Prob)$ built upon the Rips graph $R_\delta(\Set_\obs)$ with an appropriate $\delta$ is a good approximation of $D\Prob$ the persistence diagram of $\Prob$. Consequently, $D\real_\delta(\Set_\obs, \Prob)$ encodes the \textit{0th homology groups} of the underlying space of $\Set_\obs$, this is a crucial ingredient in the proof of the theoretical guarantees of \tomato.

%%%%%% previous proof, to be reorganized wiht the introduction of $\epsilon$interleaved

\begin{proof}[Proof of Theorem \ref{main_theo}]

Let denote by $\mathbf{R}_\delta=\mathbf{R}_{\delta}(\Set_\obs,\Prob)$ and $\mathbf{R}_{\sigma(\cdot)}=\mathbf{R}_{\sigma(\cdot)}(\Set_\obs,\Prob)$.
$R_\delta=R_{\delta}(\Set_\obs)$, 
and $R_{\sigma(\cdot)}=R_{\sigma(\cdot)}(\Set_\obs)$ and, for $\alpha \in \mathbb{R}$,
we set \[R_{\delta,\alpha}={R}_{\delta}\left(\Set_\obs\cap \Prob^{-1}([\alpha,+\infty])\right)\qquad\text{and}\qquad R_{\sigma(\cdot),\alpha}={R}_{\sigma(\cdot)}\left(\Set_\obs\cap \Prob^{-1}([\alpha,+\infty])\right).\]

For $\alpha\in\mathbb{R}$, let $\mathcal{C}_1,\dots, \mathcal{C}_k$ be the connected components of $R_{\delta,\alpha}$. For every $q\in\{1,\dots,k\}$, and each vertices $\obs_i,\obs_j\in \mathcal{C}_q$, we have that $\alpha_{\delta}(\obs_i,\obs_j)\geq \alpha$ and thus, by definition of $\varepsilon$, $\alpha_{\sigma(\cdot)}(\obs_i,\obs_j)\geq \alpha-\varepsilon$. Hence $\mathcal{C}_q$ is contained in a connected component of $R_{\sigma(\cdot),\alpha-\varepsilon}$. This gives a linear map: \[\varphi_\alpha\colon H_0(R_{\delta,\alpha})\to H_0(R_{\sigma(\cdot),\alpha-\varepsilon}).\]
By a similar argument, we get a linear map:
\[\psi_\alpha\colon H_0(R_{\sigma(\cdot),\alpha})\to H_0(R_{\delta,\alpha-\varepsilon}).\]
 
By construction, the following diagrams are commutative (the linear maps involved are induced by inclusions on connected components).

\resizebox{\linewidth}{!}{
 \begin{minipage}{\linewidth}
 \begin{align*}
 \xymatrix @!0 @R=2cm @C=1cm {
 &H_0(R_{\delta,\beta+\varepsilon})\ar[rrr]\ar_{\varphi_{\beta+\varepsilon}}[dl] & & & H_0(R_{\delta,\alpha-\varepsilon}) \\
 H_0(R_{\sigma(\cdot),\beta})\ar[rrrrr] & & & & & H_0(R_{\sigma(\cdot),\alpha})\ar_{\psi_{\alpha}}[ul]&} 
 &
 \xymatrix @!0 @R=2cm @C=1cm{
 &H_0(R_{\delta,\beta})\ar[rrrrr]\ar_{\varphi_{\beta}}[dl] & & & & & H_0(R_{\delta,\alpha})\ar_{\varphi_{\alpha}}[dl] \\
 H_0(R_{\sigma(\cdot),\beta-\varepsilon})\ar[rrrrr] & & & & & H_0(R_{\sigma(\cdot),\alpha-\varepsilon})} 
 \\&\\
 \xymatrix @!0 @R=2cm @C=1cm {
 H_0(R_{\delta,\beta})\ar[rrrrr] & & & & & H_0(R_{\delta,\alpha})\ar_{\varphi_{\alpha}}[dl]& \\
 & H_0(R_{\sigma(\cdot),\beta+\varepsilon})\ar_{\psi_{\beta+\varepsilon}}[ul]\ar[rrr] & & & H_0(R_{\sigma(\cdot),\alpha-\varepsilon})} 
 &
 \xymatrix @!0 @R=2cm @C=1cm {
 H_0(R_{\delta,\beta-\varepsilon})\ar[rrrrr] & & & & & H_0(R_{\delta,{\alpha}-\varepsilon}) \\
 &H_0(R_{\sigma(\cdot),\beta})\ar[rrrrr]\ar_{\psi_{\beta}}[ul] & & & & & H_0(R_{\sigma(\cdot),\alpha})\ar_{\psi_{\alpha}}[ul]}
 \end{align*}
 \end{minipage}
}
Consequently, $\mathbf{R}$ and $\mathbf R^\sigma$ are strongly $\varepsilon$-interleaved, then the bottleneck distance is bounded. 
\end{proof}

%%===========================================================================================%%
%% If you are submitting to one of the Nature Portfolio journals, using the eJP submission %%
%% system, please include the references within the manuscript file itself. You may do this %%
%% by copying the reference list from your.bbl file, paste it into the main manuscript.tex %%
%% file, and delete the associated \verb+\bibliography+ commands. %%
%%===========================================================================================%%
\section{More details on the theoretical guarantees of \tomato}\label{app:tomato}
Let us assume that $\INs$ is a Riemannian $m$-manifold and the density function $\Prob \colon \INs \to \real$ is a $\kappa$-Lipschitz probability density function with respect to the $m$-dimensional Hausdorff measure and $\kappa>0$. In order to draw the theoretical guarantees of \tomato\ we need some assumption on the persistence diagram of $\Prob$, more precisely, on the spatial distribution of the points in this diagram.

\begin{definition}
Let $d_1,d_2\in\mathcal R$ be two non-negative real numbers such that $d_1<d_2$. The persistent diagram $D\Prob$ is called \emph{$(d_1,d_2)$-separated} if every point of $D\Prob$ lies either in the region $D_1$ above the diagonal line $y=x-d_1$ or in the region $D_2$ below the diagonal line $y=x-d_2$ and to the right of the vertical line $x=d_2$.
\end{definition}

The points in the region $D_2$ will be considered as the \emph{prominence peaks} and the points in the region $D_1$ as "topological noise". 

We highlight the theoretical guarantees of \tomato. We refer the reader interested in more details and the proof of the statements to \cite{Chazal13}.

The first guarantee ensures that with reasonable assumption on the point cloud and the density $\Prob$ \tomato\ can recover the numbers of clusters induced by $\Prob$.

\begin{result}[{\cite[Theorem 9.2]{Chazal13}}]\label{R:1}
Assume that $\mathcal{S}_\obs$ is i.i.d with respect to $\Prob$. If $D\Prob$ is $(d_1,d_2)$-separated and if the parameter $\delta$ is smaller than a fraction of $d_2-d_1$ and of the convexity radius of $\mathcal{X}$, then there is a range $(d_1+2\kappa\delta,d_2-3\kappa\delta)$ of values of $\tau$ such that the number of topological regions output by $\tomato_\tau(R_\delta(S),\Prob)$ is equal to the number of peaks, of $\Prob$ with prominence at least $\tau$ with probability at least $1-e^{\Omega(n)}$ where $n$ is the number of data points. 
\end{result}

$\Omega(n)$ hides a factor increasing monotonically with $c$ and $\delta$ and depending on certain geometric quantities of the manifold $\mathcal{X}$. 

The following result tells us that, under the same hypotheses, we can recover the basins of attractions of the prominent peaks of $\Prob$. 

\begin{result}[{\cite[Theorem 10.1]{Chazal13}}]\label{R:2}
Under the same hypotheses as in Result \ref{R:1} and with the same probability, we have that, for every point $p\in D_2$, $\tomato_\tau(R_\delta(S),\Prob)$ outputs a topological region $R$ such that $R\cap \Prob^{-1}([\alpha, +\infty])=B_\tau(m_p)\cap \mathcal S\cap \Prob^{-1}([\alpha, +\infty])$ for all $\alpha\in (\alpha_\tau(p)+d_1+\frac52\kappa\delta,p_x]$, where $m_p$ is the peaks of $\Prob$ corresponding to $p$, $B_\tau(m_p)$ is the basin of attraction of $m_p$ in the underlying manifold $\mathcal{X}$, and $\alpha_\tau(m_p)$ is the first value of $\alpha$ at which $B_\tau(m_p)$ gets connected to the basin of attraction of other peaks of $\Prob$ of prominence at least $\tau$ in the superlevel-set $\Prob^{-1}([\alpha, +\infty])$.
\end{result}

In other words, $R$ is the trace of the basin of attraction $B_\tau(m_p)$ over the point of $\Set$, until the value $\alpha_\tau(m_p)$ at which the basin of attraction meets another $\tau$-prominent peak. 

Finally, it is worth mentioning that in \cite[Section 11]{Chazal13}, they also study the robustness of the approach when considering an estimation of the density function.

Theorem~\ref{main_theo} tells us that one way to ensure the bottleneck constraint in \eqref{eq:opt} is to apply the same post-processing phase used in the \tomato{} algorithm on the $\sigma$-Rips graph. It consists of applying a merging rule along the hill-climbing method on the graph with $\Prob$. This merging rule compares the \textit{topological persistence} of connected components to an additional merging parameter $\tau \in [0, \max_{\obs \in \Set_\obs} \Prob(\obs)]$ \citep{Chazal13}. 
%We will refer to this procedure as \tomato$_\tau${$(R(\Set_\obs), \Prob)$}, which returns topological regions close to the proper topological regions. 

\section{How to approximate the Purity Size objective function}\label{app:uns_crit}
For a given graph $R(\Set_\obs)$, let $\mathcal{C}_1, \dots, \mathcal{C}_k$ be the connected components of this graph, we define the  mean-sample per connected component $\mathcal{C} _q$, and the mean-sample of $\Set_\obs$ as follow:
\begin{align*}
    \mu _q &= \frac{1}{\lvert \mathcal{C} _q\rvert} \sum_{\obs \in \mathcal{C} _q} \obs, \forall q \in \{1,\dots,k\}\\
    \mu &= \frac{1}{n} \sum_{i=1}^{n} \obs _i.
\end{align*}
Then, apart from the adapted Silhouette score \eqref{silsize} that we use here the following scores can be used to approximate the purity size objective function:
\begin{itemize}
 \item Calinski-Harabasz score\\ \[S_{ch}(R(\Set_\obs)) = \left [\frac{(n-k)B}{(k-1)\sum_{q=1}^{k} W _q} \right ] \in [0, +\infty), \]
 with $B = \sum_{q=1}^k \lvert \mathcal{C} _q \rvert \| \mu _q - \mu \|^2$ is the inter-group variance, and $W _q = \sum_{\obs \in \mathcal{C} _q} \| \obs - \mu _q \|^2$ is the intra-group variance, for all $q \in \{1,\dots,k\}$. It translates that good partitioning should maximize the average inter-group variance and minimize the average intra-group variance; some well known clustering algorithms, such as K-means \citep{Lloyd82}, maximize this criterion by construction. 
 \item Davies-Bouldin score\\ \[ S_{db}(R(\Set_\obs)) = \left [\frac{1}{k} \sum_{q=1}^k \max_{j \neq q} \left ( \frac{\bar \delta _q + \bar \delta_j}{d(\mu _q, \mu_j)} \right )\right ] \in (+\infty, 0],\]
 with $\bar \delta _q = \frac{1}{\lvert \mathcal{C} _q \rvert} \sum_{\obs \in \mathcal{C} _q} d(\obs, \mu _q)$ is the average distance of all samples in the group to their mean-sample group, for all $q \in \{1,\dots,k\}$.
 \item Dunn score\\
 \[ S_{d}(R(\Set_\obs)) = \left [\frac{\min_{q, j} d(\mu _q, \mu_j)}{\max _q \Delta _q}\right ] \in [0, +\infty),\]
 with $\Delta _q = \underset{\obs, \obs' \in \mathcal{C} _q}{\max} d(\obs, \obs')$ being the diameter of group $\mathcal{C} _q$, similar to the Calinski-Harbasz score, we aim to maximize the minimum distance between the mean-sample groups and minimize the maximum group diameter. 
\end{itemize}

%% Default %%
%%\input sn-sample-bib.tex%
%\clearpage

\section{More empirical results}
\subsection{Rips graph vs $\sigma$-Rips graph} \label{apdx:graphs_comp}

In this section we show an empirical comparison in Figure~\ref{fig:all_datasets_graphs} between the Rips graph and the $\sigma$-Rips graph using the considered datasets. Overall, the $\sigma$-Rips graph achieves a better \textit{PuritySize} (PS) score than the Rips graph, except for coil-20 datasets where the scores are comparable. Note that the retrieved decision curves of the $\sigma$-Rips graphs are anti-correlated to the density estimation, which validates our intuition, and motivates the $\sigma$-Rips graph formulation.

\begin{figure}
 \centering
 \caption{Comparison study between the Rips graph and the $\sigma$-Rips graph over all datasets, the Purity Size score is reported for each minimizer.} \label{fig:all_datasets_graphs}
 \begin{subfigure}[b]{0.49\textwidth}
 \centering
 \includegraphics[width=\textwidth, trim = 0 1.5cm 0 2cm, clip=TRUE]{figs/protein_graphs}
 \label{fig:protein_graphs}
 \caption{protein}
 \end{subfigure}
 \begin{subfigure}[b]{0.49\textwidth}
 \centering
 \includegraphics[width=\textwidth, trim = 0 1.5cm 0 2cm, clip=TRUE]{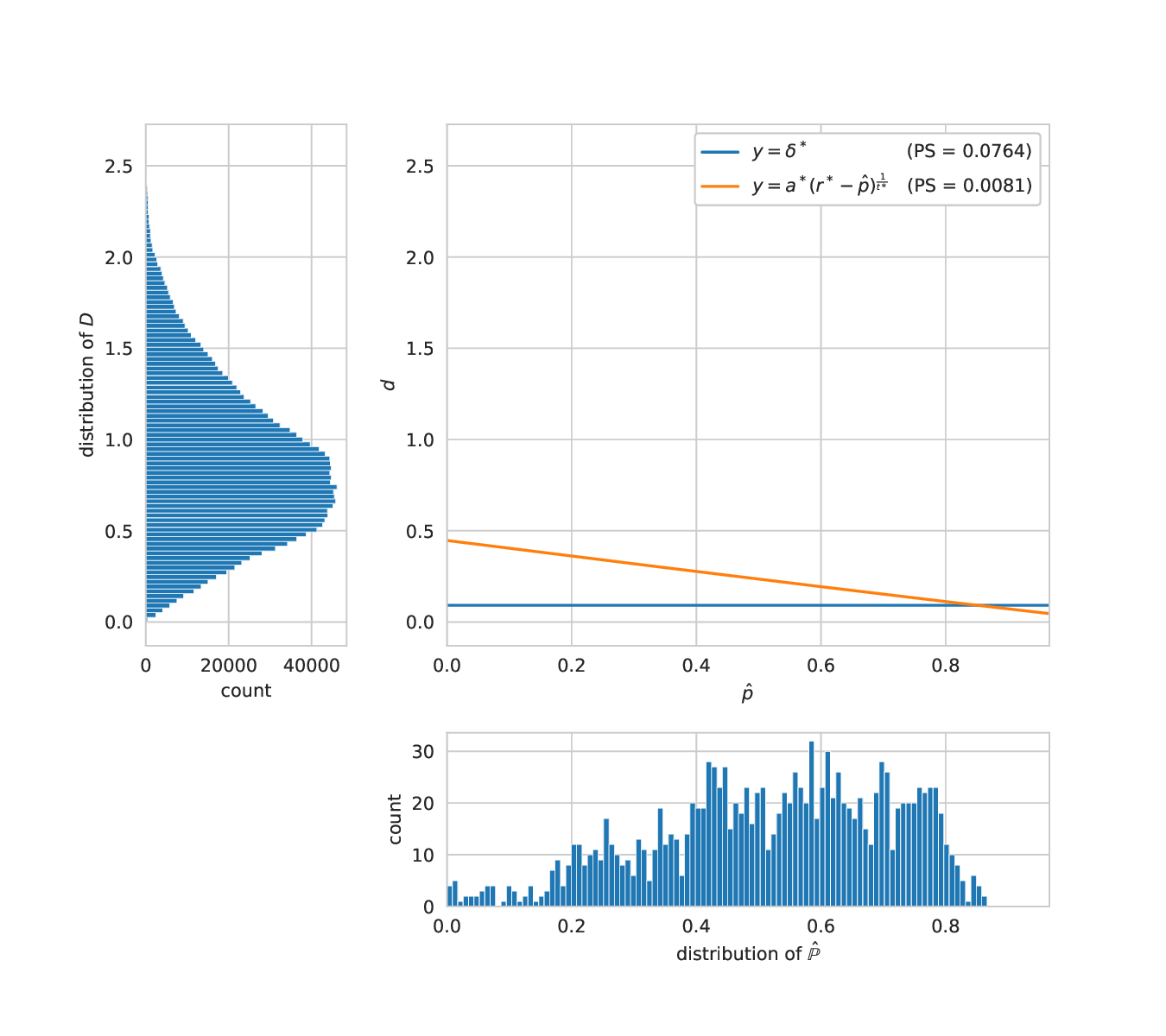}
 \label{fig:banknote_graphs}
 \caption{banknote}
 \end{subfigure}
 \begin{subfigure}[b]{0.49\textwidth}
 \centering
 \includegraphics[width=\textwidth, trim = 0 1.5cm 0 2cm, clip=TRUE]{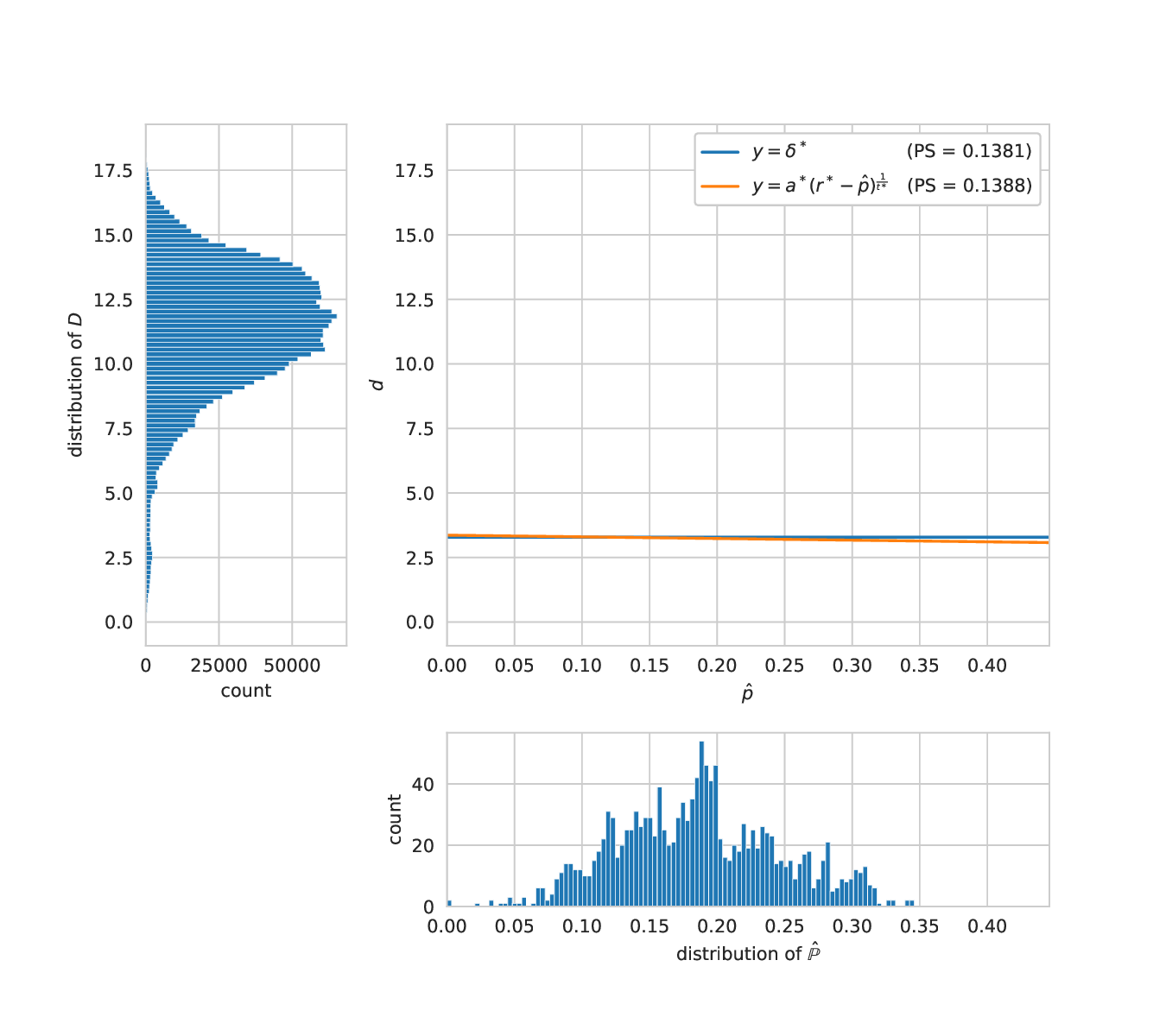}
 \label{fig:20_graphs}
 \caption{coil-20}
 \end{subfigure}
 \begin{subfigure}[b]{0.49\textwidth}
 \centering
 \includegraphics[width=\textwidth, trim = 0 1.5cm 0 2cm, clip=TRUE]{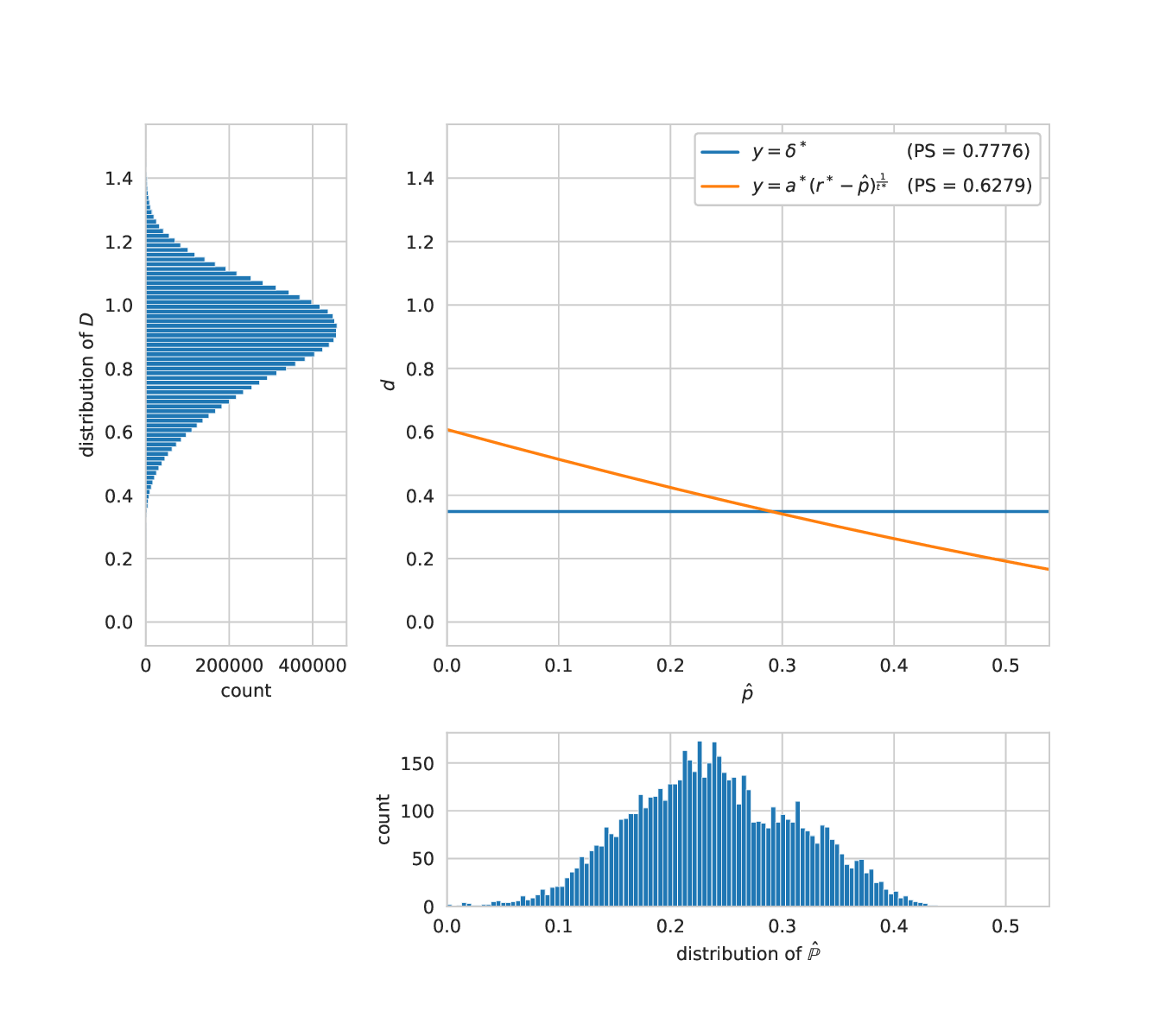}
 \label{fig:isolet_graphs}
 \caption{isolet}
 \end{subfigure}
 \begin{subfigure}[b]{0.49\textwidth}
 \centering
 \includegraphics[width=\textwidth, trim = 0 1.5cm 0 2cm, clip=TRUE]{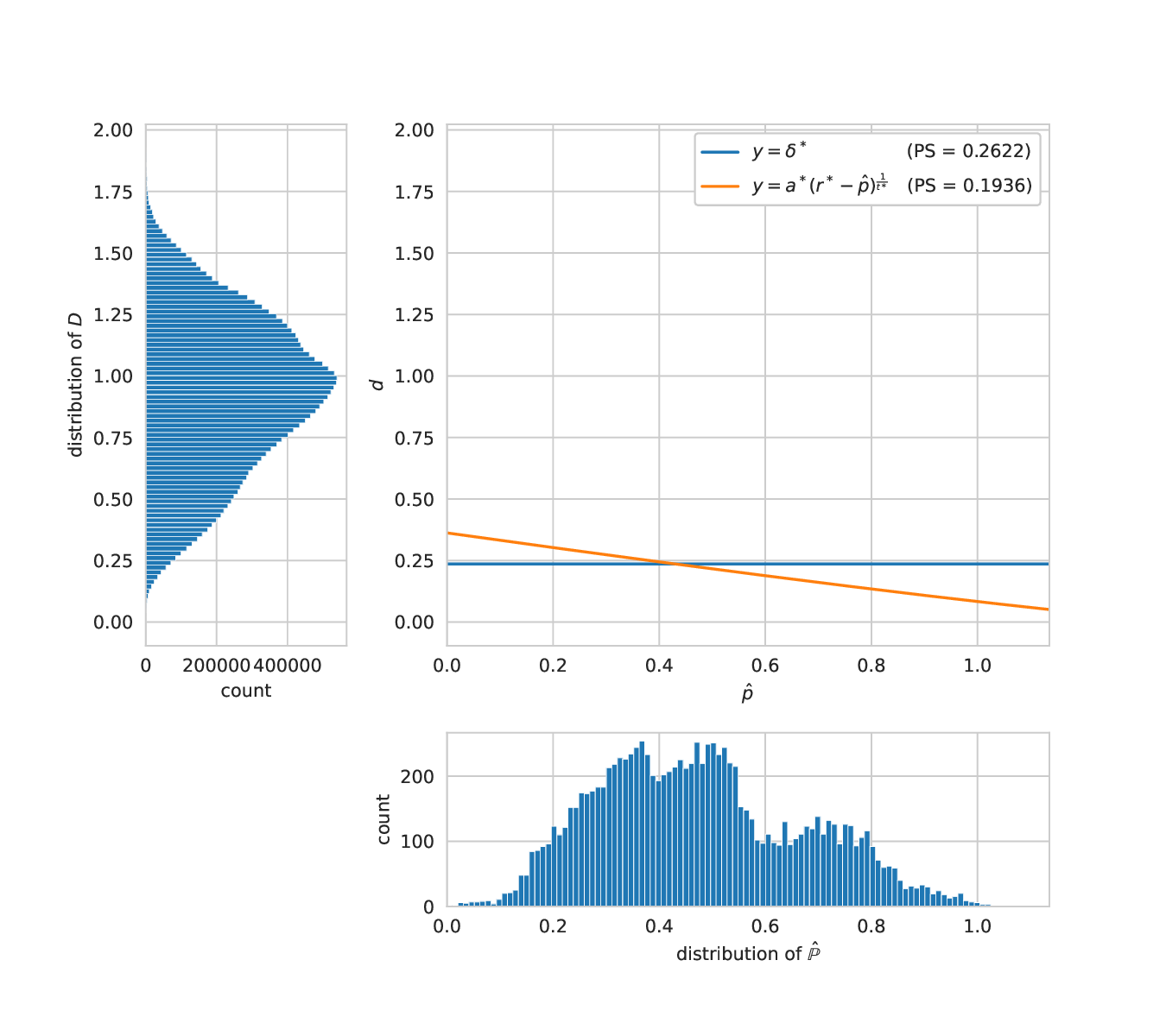}
 \label{fig:pendigits_graphs}
 \caption{pendigits}
 \end{subfigure}
 \begin{subfigure}[b]{0.49\textwidth}
 \centering
 \includegraphics[width=\textwidth, trim = 0 1.5cm 0 2cm, clip=TRUE]{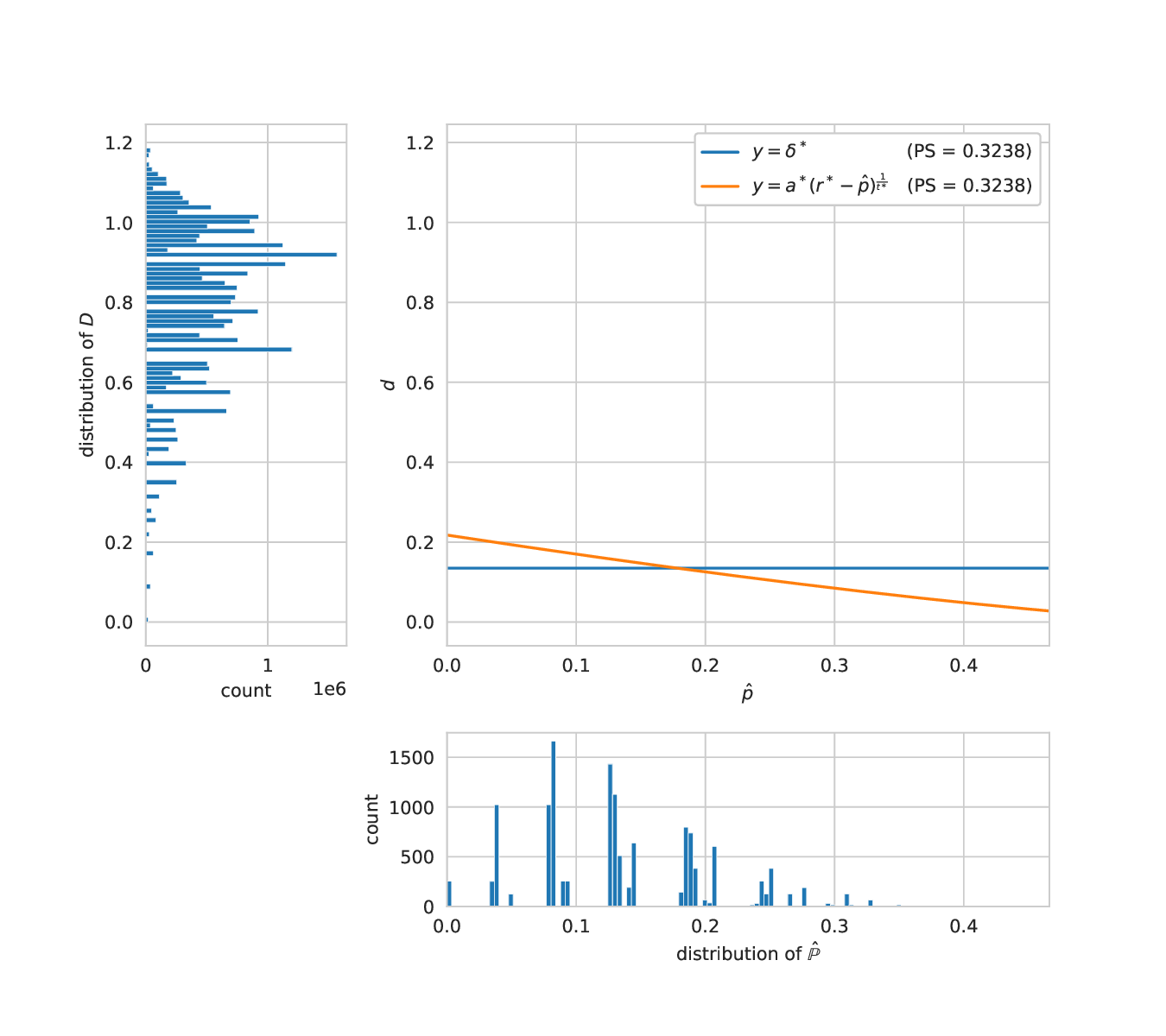}
 \label{fig:nursery_graphs}
 \caption{nursery}
 \end{subfigure}
\end{figure}

\subsection{Cold-start results}\label{apdx:cold_start_comp}
This section presents the numerical results obtained in our empirical investigation of the cold-start problem in Table~\ref{tab:res} for several budgets. It shows that the proposed PTR approach achieves competitive performance compared to the baseline methods, we also notice that most of the methods except for the k-means-based methods and PTR suffer from degraded performance compared to the random selection strategy.

\begin{table}[t]
 \caption{Average balanced classification accuracy (in \%) and standard deviation of random forest classifier with the initial training set obtained from different methods over 20 stratified random splits for different budgets $\mathcal{B}$. $^\uparrow/^\downarrow$ indicate statistically significantly better/worse performance than Random Selection RS, according to a Wilcoxon rank sum test ($p < 0.05$) \citep{Lehman75}.}
 \label{tab:res}
 \begin{adjustbox}{max width=\textwidth}
 \centering
 \begin{tabular}{llllllllll}
 \toprule
 Dataset & $\mathcal{B}$ & RS & KM & KM+ME & Km & AHC & FFT & APC & PTR\\
 \midrule
 \multirow{3}{*}{protein} & 3 & $16.9$ & $21.2^\uparrow$ & $\mathbf{23.9}^\uparrow$ & $21.2^\uparrow$ & $22.7^\uparrow$ & $17.4$ & $16.7$ & $22.1^\uparrow$\\
 && $(4.0)$ & $(1.8)$ & $(2.2)$ & $(4.4)$ & $(2.5)$ & $(3.3)$ & $(3.3)$ & $(6.2)$\\
 & 20 & $36.4$ & $42.1^\uparrow$ & $45.5^\uparrow$ & $39.2$ & $43.4^\uparrow$ & $26.1^\downarrow$ & $39.2$ & $\mathbf{54.0}^\uparrow$\\
 && $(3.8)$ & $(3.9)$ & $(2.5)$ & $(4.4)$ & $(3.4)$ & $(3.4)$ & $(3.7)$ & $(3.4)$\\
 \midrule
 \multirow{3}{*}{banknote} & 3 & $55.5$ & $74.0^\uparrow$ & $\mathbf{84.3}^\uparrow$ & $62.5^\uparrow$ & $63.7^\uparrow$ & $58.2^\uparrow$ & $58.7$ & $70.2^\uparrow$\\
 && $(7.2)$ & $(4.6)$ & $(5.6)$ & $(3.3)$ & $(4.5)$ & $(7.3)$ & $(8.0)$ & $(14.7)$\\
 & 20 & $87.6$ & $90.7^\uparrow$ & $92.4^\uparrow$ & $92.3^\uparrow$ & $92.6^\uparrow$ & $71.9^\downarrow$ & $90.9^\uparrow$ & $\mathbf{93.9}^\uparrow$\\
 && $(2.9)$ & $(2.4)$ & $(2.0)$ & $(2.4)$ & $(2.9)$ & $(7.2)$ & $(3.2)$ & $(3.4)$\\
 \midrule
 \multirow{3}{*}{coil-20} & 3 & $12.6$ & $\mathbf{15.0}^\uparrow$ & $\mathbf{15.0}^\uparrow$ & $\mathbf{15.0}^\uparrow$ & $\mathbf{15.0}^\uparrow$ & $10.8^\downarrow$ & $11.7$ & $13.6$\\
 && $(2.6)$ & $(0.0)$ & $(0.0)$ & $(0.0)$ & $(0.0)$ & $(2.0)$ & $(2.3)$ & $(1.7)$\\
 & 20 & $42.0$ & $56.7^\uparrow$ & $63.0^\uparrow$ & $42.3$ & $58.1^\uparrow$ & $25.6^\downarrow$ & $41.4$ & $\mathbf{71.1}^\uparrow$\\
 && $(5.8)$ & $(3.7)$ & $(2.8)$ & $(3.5)$ & $(4.1)$ & $(2.5)$ & $(4.7)$ & $(3.8)$\\
 \midrule
 \multirow{3}{*}{isolet} & 3 & $07.6$ & $08.7^\uparrow$ & $09.7^\uparrow$ & $07.8$ & $09.1^\uparrow$ & $09.2^\uparrow$ & $07.5$ & $\mathbf{10.8}^\uparrow$\\
 && $(1.5)$ & $(0.9)$ & $(0.6)$ & $(1.6)$ & $(1.9)$ & $(1.0)$ & $(1.8)$ & $(1.1)$\\
 & 20 & $19.2$ & $27.9^\uparrow$ & $\mathbf{40.4}^\uparrow$ & $10.7^\downarrow$ & $28.2^\uparrow$ & $18.8$ & $21.1^\uparrow$ & $38.6^\uparrow$\\
 && $(2.7)$ & $(2.5)$ & $(3.2)$ & $(2.0)$ & $(2.1)$ & $(2.4)$ & $(3.1)$ & $(3.2)$\\
 \midrule
 \multirow{3}{*}{pendigits} & 3 & $21.5$ & $21.3$ & $22.5$ & $26.6^\uparrow$ & $19.4^\downarrow$ & $17.3^\downarrow$ & $17.8^\downarrow$ & $\mathbf{29.9}^\uparrow$\\
 && $(3.5)$ & $(1.9)$ & $(2.1)$ & $(2.6)$ & $(1.8)$ & $(3.7)$ & $(4.9)$ & $(0.0)$\\
 & 20 & $54.3$ & $72.3^\uparrow$ & $75.8^\uparrow$ & $64.0^\uparrow$ & $72.3^\uparrow$ & $34.8^\downarrow$ & $52.2$ & $\mathbf{87.7}^\uparrow$\\
 && $(5.9)$ & $(2.7)$ & $(2.3)$ & $(3.6)$ & $(2.5)$ & $(4.5)$ & $(5.9)$ & $(4.1)$\\
 \midrule
 \multirow{3}{*}{nursery} & 3 & $30.7$ & $29.2$ & $30.2$ & $25.0^\downarrow$ & $28.3^\downarrow$ & $30.0$ & $30.0$ & $\mathbf{35.1}^\uparrow$\\
 && $(4.0)$ & $(5.2)$ & $(6.5)$ & $(0.2)$ & $(3.9)$ & $(3.2)$ & $(3.7)$ & $(5.6)$\\
 & 20 & $\mathbf{55.3}$ & $52.8^\downarrow$ & $54.4$ & $32.9^\downarrow$ & $53.8$ & $39.8^\downarrow$ & $52.5^\downarrow$ & $54.1$\\
 && $(2.8)$ & $(3.3)$ & $(3.0)$ & $(1.1)$ & $(2.7)$ & $(1.1)$ & $(4.9)$ & $(4.5)$\\
 \bottomrule
 \end{tabular}
 \end{adjustbox}
\end{table}

\subsection{Active learning results}\label{Fig:poolBased}
This last section illustrates the empirical results of pool-based active learning for the rest of the considered datasets using our proposed approach, compared against the random sampling strategy. Figures \ref{fig:protein_banknote}, \ref{fig:coil-20_isolet} and \ref{fig:pendigits_nursery}(a) show a clear and significant gain of the PTR methodology for pool-based active learning strategies against random selection. The results in Figure \ref{fig:pendigits_nursery}(b) indicate that the PTR approach is not robust against the high label imbalance datasets. We observed that the propagation step over the PTR tends to amplify the class imbalance in the resulting training set.

\begin{figure}[t]
\centering
 \caption{Average balanced classification accuracy and standard deviation of different pool-based active learning strategies and budgets on protein and banknote datasets, using random forest estimator over 20 stratified random splits.}
 \label{fig:protein_banknote}
 \begin{subfigure}[b]{1\textwidth}
 \centering
 \includegraphics[width=\textwidth]{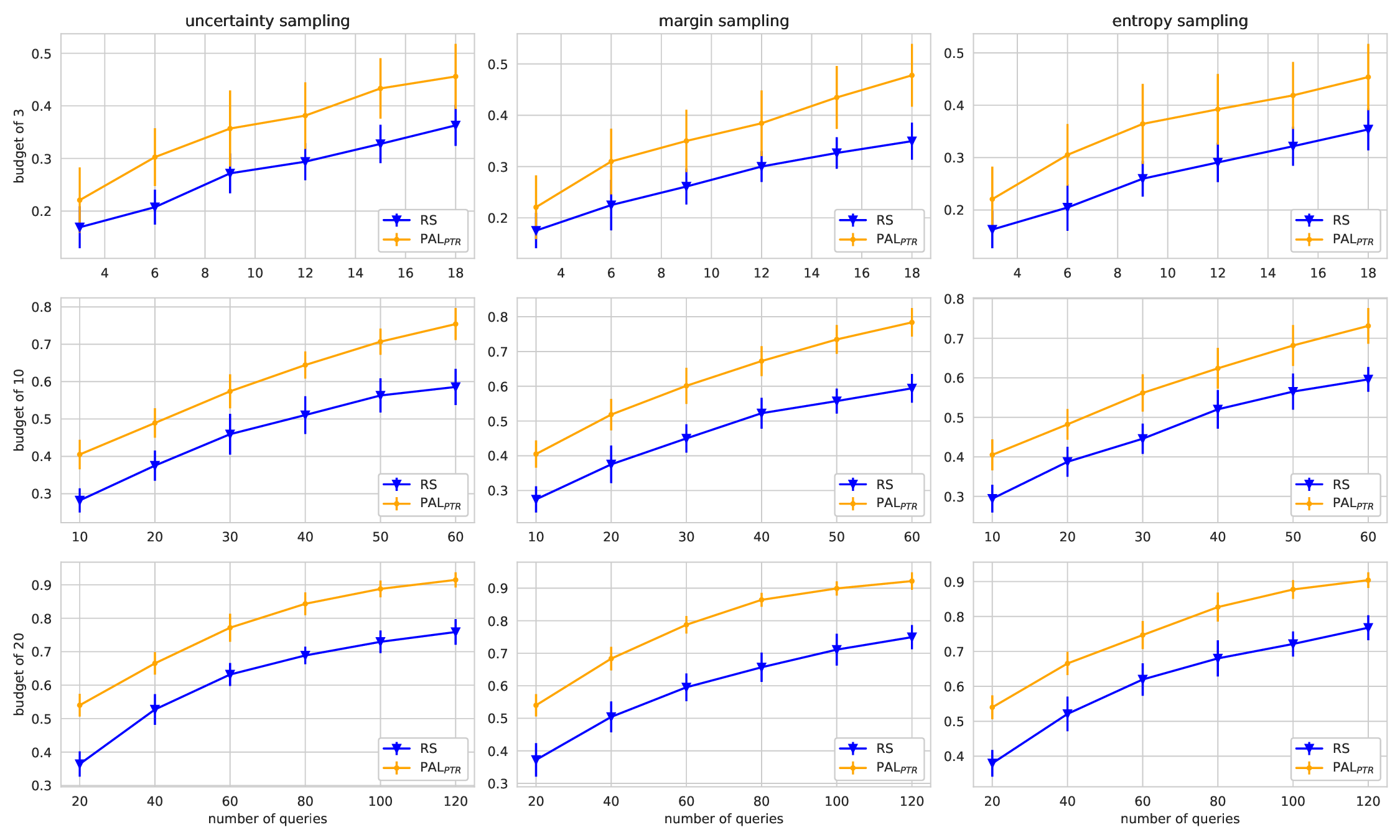}
 \caption{protein}
 \end{subfigure}
 \begin{subfigure}[b]{1\textwidth}
 \centering
 \includegraphics[width=\textwidth]{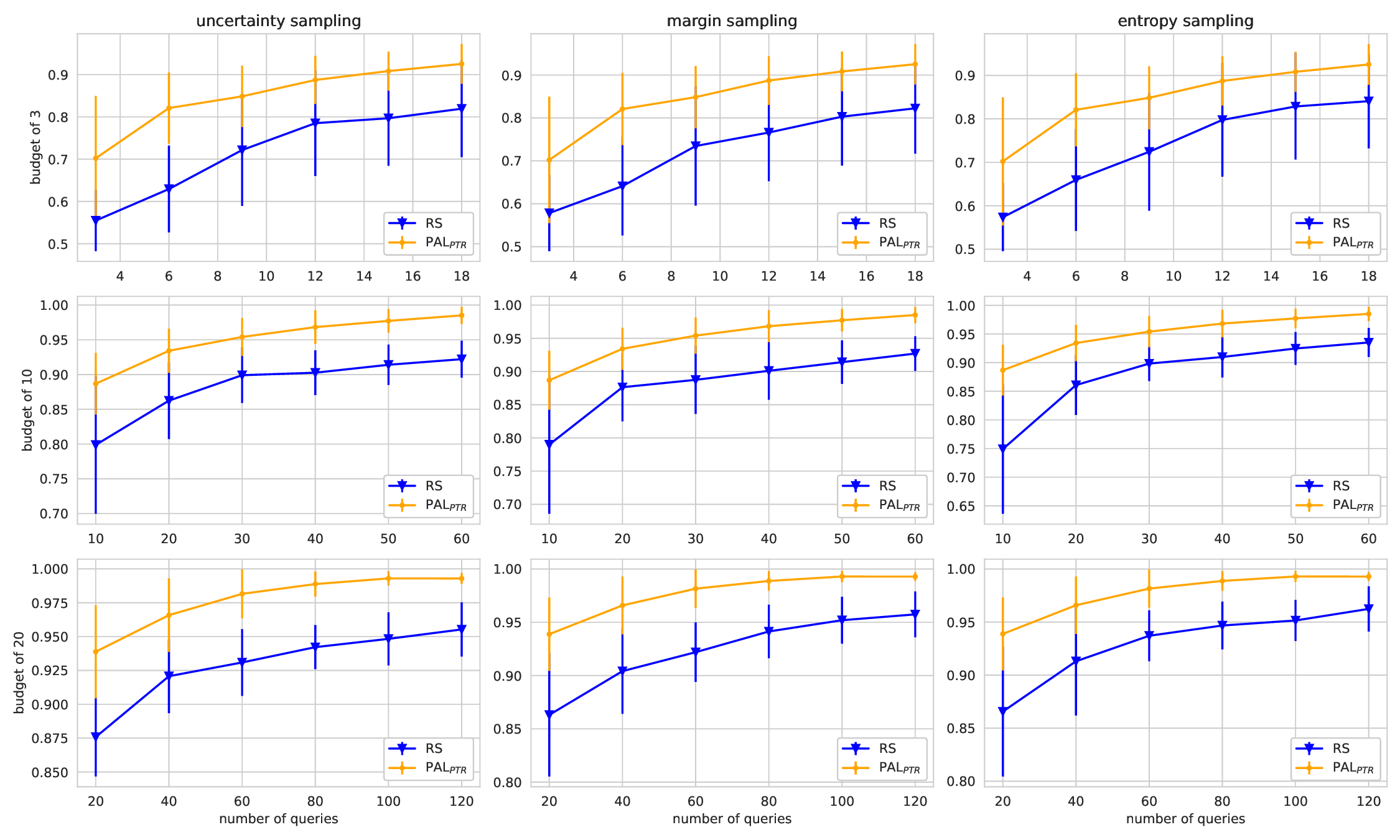}
 \caption{banknote}
 \end{subfigure}
\end{figure}

\begin{figure}[t]
 \centering
 \caption{Average balanced classification accuracy and standard deviation of different pool-based active learning strategies and budgets on coil-20 and isolet datasets, using random forest estimator over 20 stratified random splits.}
 \label{fig:coil-20_isolet}
 \begin{subfigure}[b]{1\textwidth}
 \centering
 \includegraphics[width=\textwidth]{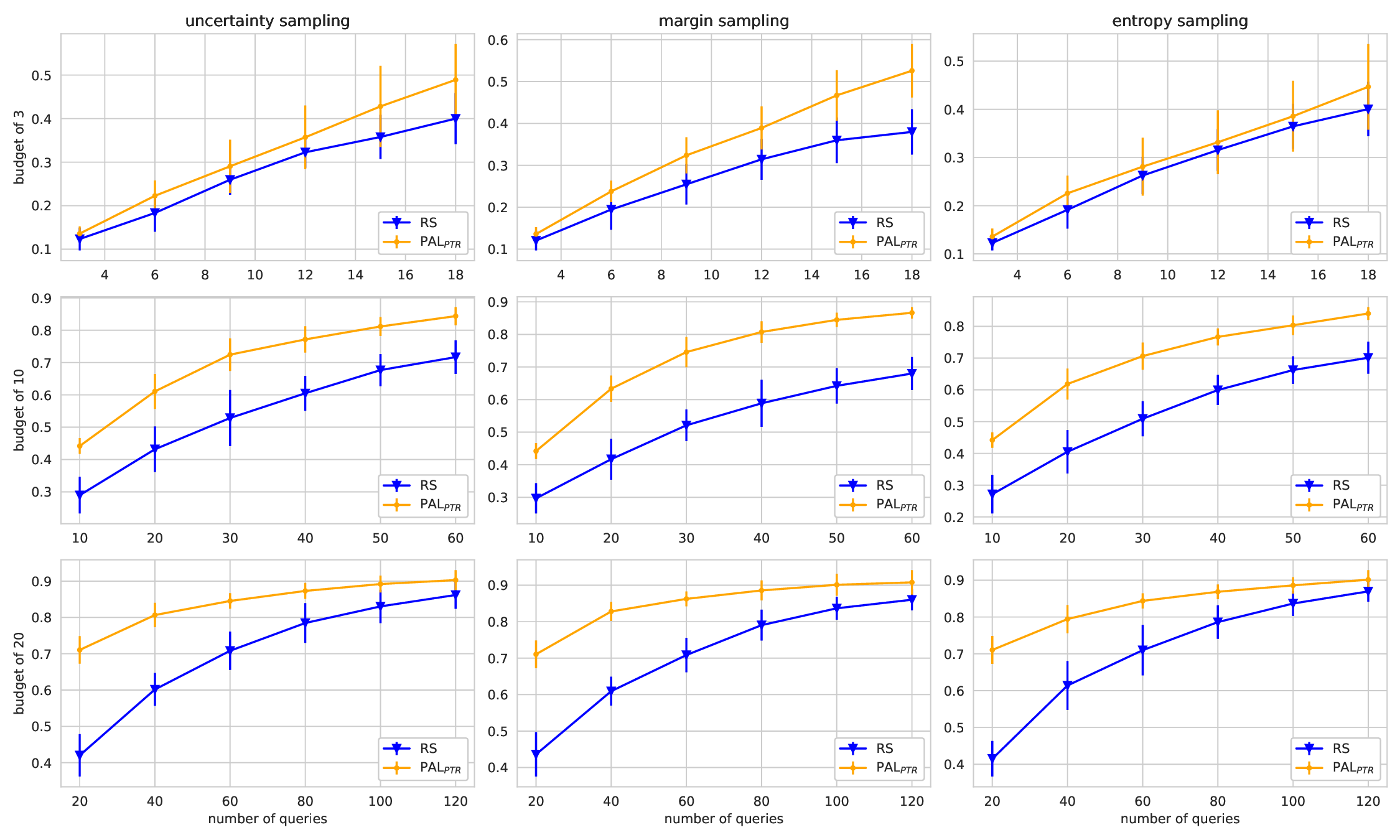}
 \caption{coil20}
 \end{subfigure}
 \begin{subfigure}[b]{1\textwidth}
 \centering
 \includegraphics[width=\textwidth]{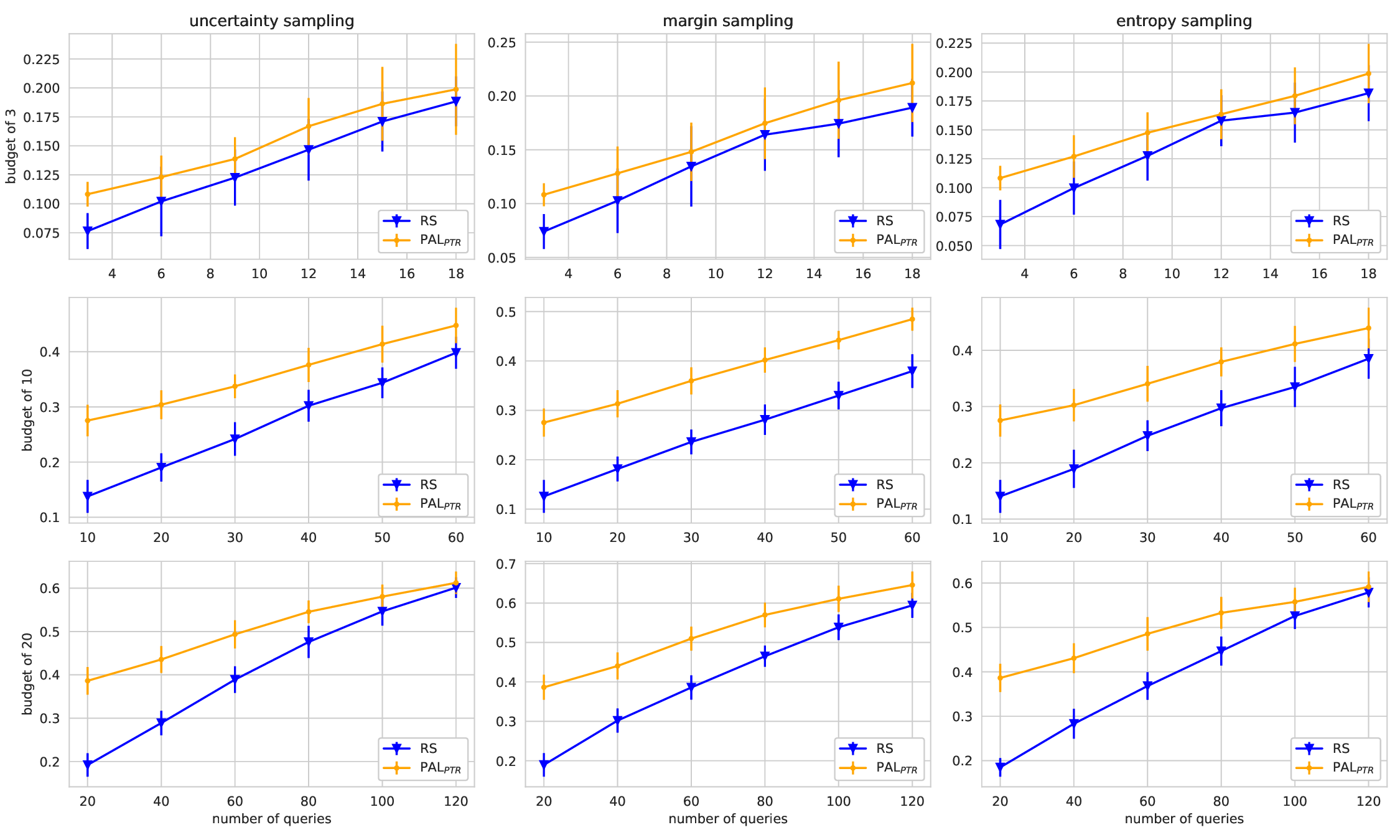}
 \caption{isolet}
 \end{subfigure}
\end{figure}

\begin{figure}
 \centering
 \caption{Average balanced classification accuracy and standard deviation of different pool-based active learning strategies and budgets on pendigits and nursery datasets, using random forest estimator over 20 stratified random splits.} \label{fig:pendigits_nursery}
 \begin{subfigure}[b]{1\textwidth}
 \centering
 \includegraphics[width=\textwidth]{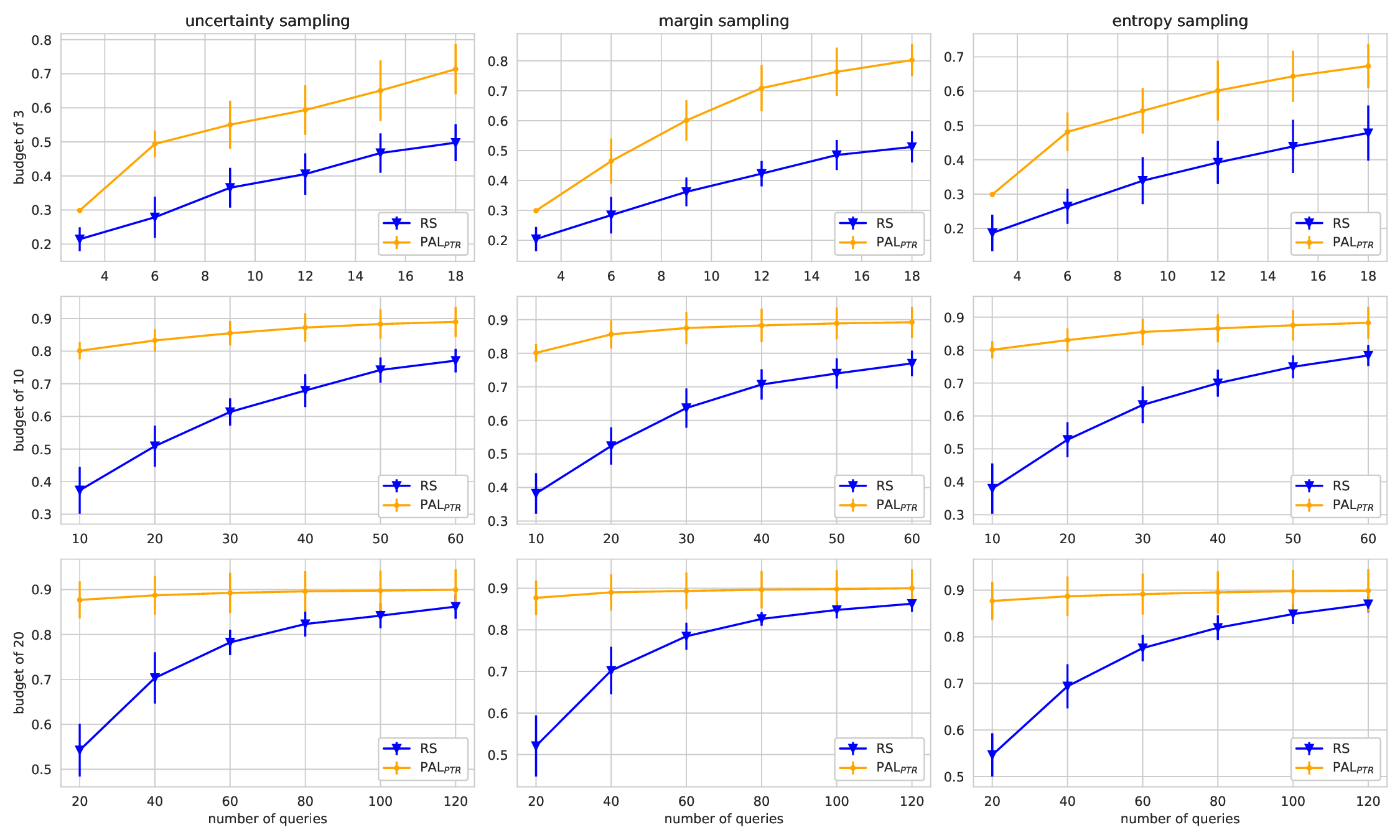}
 \caption{pendigits}
 \end{subfigure}
 \begin{subfigure}[b]{1\textwidth}
 \centering
 \includegraphics[width=\textwidth]{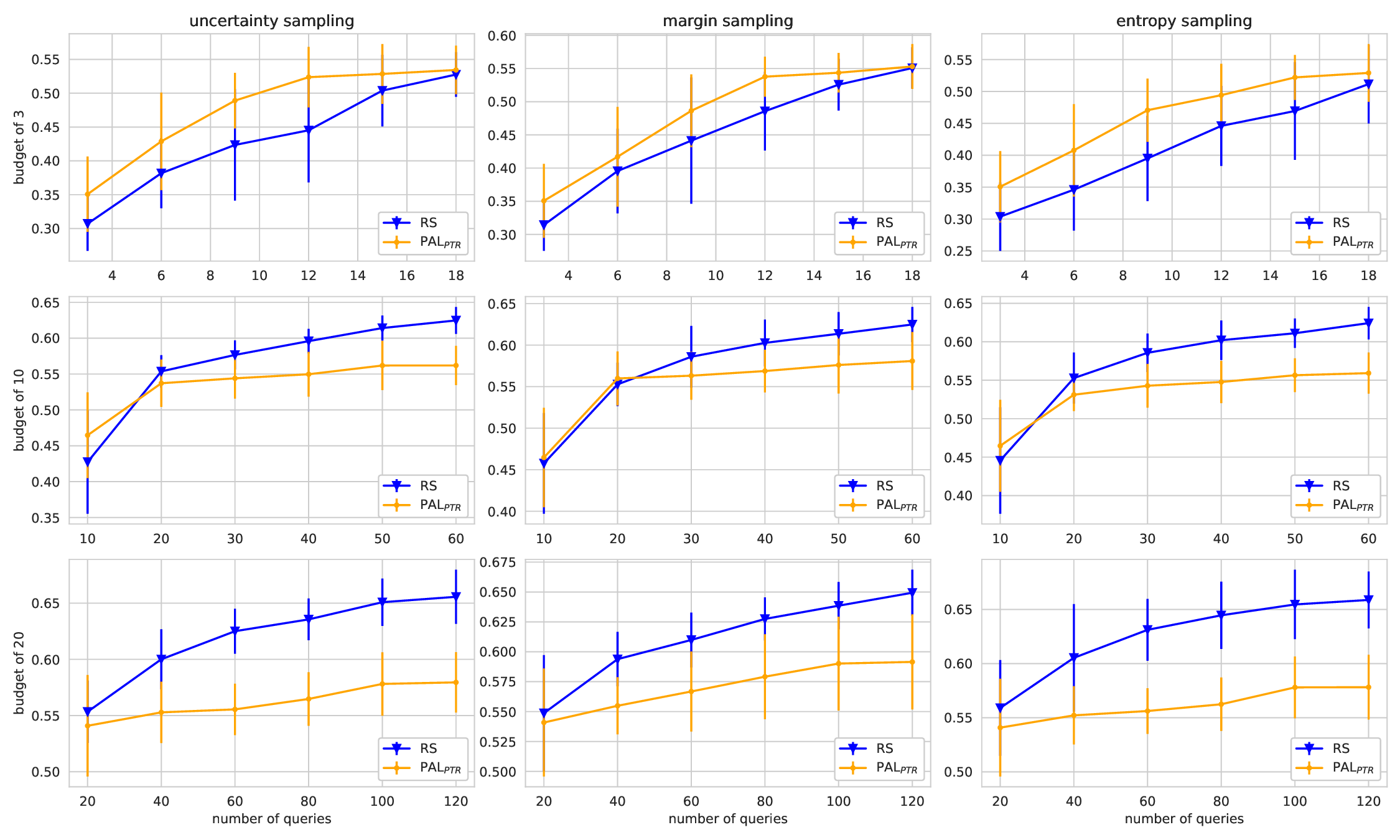}
 \caption{nursery}
 \end{subfigure}
\end{figure}
\end{document}